\documentclass{article}

\usepackage[accepted]{icml2023}

\usepackage{amsmath}
\usepackage{amssymb}
\usepackage{mathtools}
\usepackage{amsthm}

\usepackage{hyperref}       %
\usepackage[capitalize,noabbrev]{cleveref}

\usepackage{times}

\usepackage[utf8]{inputenc} %
\usepackage[T1]{fontenc}    %
\usepackage{url}            %
\usepackage{booktabs}       %
\usepackage{amsfonts}       %
\usepackage{enumitem}
\usepackage{nicefrac}       %
\usepackage{microtype}      %
\usepackage{xcolor}%
\usepackage{graphicx}
\usepackage{cancel}
\usepackage{natbib}
\usepackage{pifont}
\usepackage{subcaption}
\usepackage{wrapfig}
\usepackage{bbm}

\newenvironment{hproof}{%
  \proof}{\endproof}

\theoremstyle{plain}
\newtheorem{theorem}{Theorem}[section]

\newtheorem{lemma}[theorem]{Lemma}

\theoremstyle{definition}

\theoremstyle{remark}

\usepackage{amsmath,amsfonts,bm}

\newcommand{\figtop}{{\em (Top)}}
\newcommand{\figbottom}{{\em (Bottom)}}

\def\eqref#1{equation~\ref{#1}}

\def\1{\bm{1}}

\def\rvy{{\mathbf{y}}}

\DeclareMathAlphabet{\mathsfit}{\encodingdefault}{\sfdefault}{m}{sl}
\SetMathAlphabet{\mathsfit}{bold}{\encodingdefault}{\sfdefault}{bx}{n}

\def\gA{{\mathcal{A}}}

\def\gL{{\mathcal{L}}}

\def\gO{{\mathcal{O}}}

\def\gS{{\mathcal{S}}}

\newcommand{\E}{\mathbb{E}}

\DeclareMathOperator*{\argmax}{arg\,max}
\DeclareMathOperator*{\argmin}{arg\,min}

\icmltitlerunning{A Connection between One-Step RL and Critic Regularization in RL}

\begin{document}

\twocolumn[
\icmltitle{A Connection between One-Step RL and \\ Critic Regularization in Reinforcement Learning}

\icmlsetsymbol{equal}{*}

\begin{icmlauthorlist}
\icmlauthor{Benjamin Eysenbach}{goo,cmu}
\icmlauthor{Matthieu Geist}{goo}
\icmlauthor{Sergey Levine}{goo,berk}
\icmlauthor{Ruslan Salakhutdinov}{cmu}
\end{icmlauthorlist}

\icmlaffiliation{cmu}{Carnegie Mellon University}
\icmlaffiliation{goo}{Google Research}
\icmlaffiliation{berk}{UC Berkeley}

\icmlcorrespondingauthor{Benjamin Eysenbach}{beysenba@cs.cmu.edu}

\icmlkeywords{reinforcement learning, offline RL, regularization}

\vskip 0.3in
]

\printAffiliationsAndNotice{}  %

\begin{abstract}
As with any machine learning problem with limited data, effective offline RL algorithms require careful regularization to avoid overfitting. One class of methods, known as one-step RL, perform just one step of policy improvement. These methods, which include advantage-weighted regression and conditional behavioral cloning, are thus simple and stable, but can have limited asymptotic performance. A second class of methods, known as critic regularization, perform many steps of policy improvement with a regularized objective. These methods typically require more compute but have appealing lower-bound guarantees. In this paper, we draw a connection between these methods: applying a multi-step critic regularization method with a regularization coefficient of 1 yields the same policy as one-step RL. While our theoretical results require assumptions (e.g., deterministic dynamics), our experiments nevertheless show that our analysis makes accurate, testable predictions about practical offline RL methods (CQL and one-step RL) with commonly-used hyperparameters.

\end{abstract}

\section{Introduction}

\begin{figure}[ht]
\centering
\includegraphics[width=\linewidth]{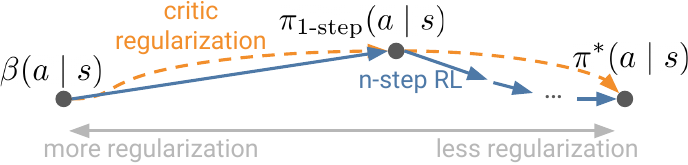}%
\vspace{-0.2em}
\caption{ Both $n$-step RL and critic regularization can interpolate between behavioral cloning (left) and un-regularized RL (right) by varying the regularization parameter. Endpoints of these regularization paths are the same. We prove that these methods also obtain the same policy for an intermediate degree of regularization.}
\label{fig:method}
\vspace{-1.5em}
\end{figure}

Reinforcement learning (RL) algorithms tend to perform better when regularized, especially when given access to only limited data, and especially in batch (i.e., offline) settings where the agent is unable to collect new experience.  While RL algorithms can be regularized using the same tools as in supervised learning (e.g., weight decay, dropout), we will use ``regularization'' to refer to those techniques unique to the RL setting. Such regularization methods include policy regularization (penalizing the policy for sampling out-of-distribution action) and value regularization (penalizing the critic for making large predictions).
Research on these sorts of regularization has grown significantly in recent years, yet theoretical work studying the tradeoffs between regularization methods remains limited~\citep{vieillard2020leverage}.

Many RL methods perform regularization and can be classified by whether they perform one or many steps of policy improvement.
\emph{One-step RL} methods~\citep{brandfonbrener2021offline, peng2019advantage, peters2007reinforcement, peters2010relative} perform one step of policy iteration, updating the policy to choose actions the are best according to the Q-function of the behavioral policy. The policy is often regularized to not deviate far from the behavioral policy.
In theory, policy iteration can take a large number of iterations ($\tilde{\gO}(|\gS||\gA| / (1 - \gamma))$~\citep{scherrer2013improved}) to converge, so one-step RL (one step of policy iteration) fails to find the optimal policy on most tasks. Empirically, policy iteration often converges in a smaller number of iterations~\citep[Sec.~4.3]{sutton2018reinforcement}, and the policy after just a single iteration can sometimes achieve performance comparable to multi-step RL methods~\citep{brandfonbrener2021offline}.
\emph{Critic regularization} methods modify the training of the value function such that it predicts smaller returns for unseen actions~\citep{kumar2020conservative, chebotar2021actionable, yu2021combo, hatch2022example, nachum2019algaedice, an2021uncertainty, bai2022pessimistic, buckman2020importance}.
Intuitively, such critic regularization causes the policy to avoid sampling unseen actions.
In this paper, we will use ``critic regularization'' to specifically refer to multi-step methods that use critic regularization; multi-step methods that do not utilize critic regularization (e.g., KL control~\citep{ziebart2010modeling}, TD3+BC~\citep{fujimoto2021minimalist}) are outside the scope of our analysis.

These RL regularization methods appear distinct. Critic regularization typically involves solving a two-player game, whereby a policy predicts actions with high values while the critic decreases the values predicted for those actions. Prior work~\citep{kumar2020conservative} has argued that this complexity is worthwhile because the regularization effect is being propagated across time via Bellman backups: decreasing the value at one state will also decrease the value at states that lead there.

In this paper, we show that a certain type of one-step RL is equivalent to a certain type of critic regularization, under some assumptions (see Fig.~\ref{fig:method}).
The key idea is that, when using a certain TD loss, the regularized critic updates converge not to the true Q-values, but rather the Q-values multiplied by an importance weight. For the critic, these importance weights mean that the Q-values end up estimating the expected returns of the behavioral policy ($Q^\beta$, as in many one-step methods~\citep{peters2010relative, peters2007reinforcement, peng2019advantage, brandfonbrener2021offline}), rather than the expected returns of the optimal policy ($Q^\pi$). For the actor, these importance weights mean that the logarithm of the Q-values includes a term that looks like a KL divergence. 
This connection allows us to make precise how critic regularization methods implicitly regularize the policy.

The key contribution of this paper is a construction showing that one-step RL produces the same policy as a multi-step critic regularization method, for a certain regularization coefficient.
We also discuss similar connections in settings with varying degrees of regularization, in goal-conditioned settings, and in RL settings with success classifiers. The main assumption
behind these results is that the critic is updated using an update rule based on the cross entropy loss, rather than the MSE loss. Our result is potentially surprising because, algorithmically and mechanistically, one-step RL and critic regularization are very different. Nonetheless, our analysis may help explain why prior work has found that one-step RL and critic regularization methods can perform similarly on some ~\citep{brandfonbrener2021offline, emmons2021rvs} (but not all~\citep{kostrikov2021offline}) problems.
Our results hint that one-step RL methods may be a simpler approach to achieving the theoretical guarantees typically associated with critic regularization methods.
Our analysis does not say that practical implementations (which violate our assumptions) will always behave the same in practice; they do not~\citep[Sec.~7]{brandfonbrener2021offline}. Nonetheless, our experiments show that our analysis makes accurate testable predictions about practical methods.
While our results do not say whether users should regularize the actor or critic in practice, they hint that one-step RL methods may be a simpler way of achieving the theoretical and empirical properties of critic regularization on RL tasks that require strong regularization.

\section{Related Work}

Regularization has been applied to RL in many different ways~\citep{neu2017unified, geist2019theory}, and features prominently in offline RL methods~\citep{lange2012batch, levine2020offline}.
While RL algorithms can be regularized using the same techniques as in supervised learning (e.g., weight decay, dropout), our focus will be on regularization methods unique to the RL setting.
Such RL-specific regularization methods can be roughly categorized into one-step RL methods and critic regularization methods.

One-step RL methods~\citep{brandfonbrener2021offline, gulccehre2020rl, peters2007reinforcement, peng2019advantage, peters2010relative, wang2018exponentially} apply a single step of policy improvement to the behavioral policy. These methods first estimate the Q-values of the behavioral policy, either via regression or iterative Bellman updates. Then, these methods optimize the policy to maximize these Q-values minus an actor regularizer.
Many goal-conditioned or task-conditioned imitation learning methods~\citep{savinov2018semi, ding2019goal, sun2019policy, ghosh2020learning, paster2020planning, Yang2021MHERMH, srivastava2019training, kumar2019reward, chen2021decision, lynch2021language, li2020generalized, eysenbach2020rewriting} also fits into this mold~\citep{eysenbach2022imitating}, yielding policies that maximize the Q-values of the behavioral policy while avoiding unseen actions. Note that non-conditional imitation learning methods do not perform policy improvement, and do not fit into this mold. One-step methods are typically simple to implement and computationally efficient. \looseness=-1

Critic regularization methods instead modify the objective for the Q-function so that it predicts lower returns for unseen actions~\citep{kumar2020conservative, chebotar2021actionable, yu2021combo, hatch2022example, nachum2019algaedice, an2021uncertainty, bai2022pessimistic, buckman2020importance}.
Critic regularization methods are typically more challenging to implement correctly and more computationally demanding~\citep{kumar2020conservative, nachum2019algaedice, bai2022pessimistic, an2021uncertainty}, but can lead to better results on some challenging problems~\citep{kostrikov2021offline}
Our analysis will show that one-step RL is equivalent to a certain type of critic regularization.

Some regularization methods do not fit exactly into these two categories. Methods like KL control regularize both the actor and the reward function~\citep{geist2019theory, ziebart2010modeling, haarnoja2018soft, abdolmaleki2018maximum, wu2019behavior, jaques2019way, rezaeifar2022offline}. Other methods only regularize the policy used in the critic updates~\citep{fujimoto2019off, kumar2019stabilizing}.

Our results are conceptually similar to prior work on regularization in the supervised learning setting. There, regularization methods like weight decay, spectral regularization, and early stopping all appear quite distinct, but are actually mathematical equivalent under some assumptions~\citep{bakushinskii1967general, wahba1987three, fleming1990equivalence, santos1996equivalence, bauer2007regularization}. This result is surprising because the methods are so different: weight decay modifies the loss function, early stopping modifies the update procedure, and spectral normalization is a post-hoc correction.

\section{Preliminaries}

We start by defining the single-task RL problem, and then introduce prototypical examples of one-step RL and critic regularization. We then define an actor critic algorithm for use in our analysis.

\subsection{Notation}
\label{sec:notation}
We assume an MDP with states $s$, actions $a$, initial state distribution $p_0(s_0)$, dynamics $p(s' \mid s, a)$, and reward function $r(s, a)$. %
We assume episodes always have infinite length (i.e., there are no terminal states). Without loss of generality, we assume rewards are positive; adding a positive constant to all rewards can make them all positive without changing the optimal policy. We will learn a Markovian policy $\pi(a \mid s)$ to maximize the expected discounted sum of rewards:
\begin{equation*}
    \max_\pi \E_{\pi(\tau)}\left[\sum_{t=0}^\infty \gamma^t r(s_t, a_t) \mid s_0 \sim p_0(s_0) \right],
\end{equation*}
where $\pi(\tau) = p(s_0)\prod_{t=0}^\infty \pi(a_t \mid s_t) p(s_{t+1} \mid s_t, a_t)$  is the probability of policy $\pi$ sampling an infinite-length trajectory $\tau = (s_0, a_0, \cdots)$.
We define Q-values for policy $\pi(a \mid s)$ as 
\begin{equation*}
Q^{\pi}(s, a) = \E_{\pi(\tau)}\left[\sum_{t=0}^\infty \gamma^t r(s_t, a_t) \mid s_0 = s, a_0 = a \right].
\end{equation*}
Because the rewards are positive, these Q-values are also positive, $Q^\pi(s, a) > 0$. Since we focus on the offline setting, we will consider two policies: $\beta(a \mid s)$ is the \emph{behavioral} policy that collected the dataset, and $\pi(a \mid s)$ is the \emph{online} policy output by the algorithm that attempts to maximize the rewards. We will use $p(s, a, s')$ to denote the empirical distribution of transitions in an offline dataset, and $p(s, a)$ and $p(s)$ denote the corresponding marginal distributions. The behavioral policy is defined as $\beta(a \mid s) = p(a \mid s)$.

\subsection{Examples of Regularization in RL}
While actor and critic regularization methods can be implemented in many ways, we introduce two prototypical examples below to make our discussion more concrete.

\paragraph{Example of one-step RL:~\citet{brandfonbrener2021offline}.}
One-step RL first estimates the Q-values of the behavioral policy ($Q^\beta(s, a)$), and then optimizes the policy to maximize the Q-values minus a actor regularizer. While the actor regularizer can take different forms and the Q-values can be learned via regression, we will use a reverse KL regularizer and TD-style critic update so that the objective is similar to critic regularization:
\begin{align}
    & \max_\pi  \E_{p(s)\pi(a \mid s)}\left[ Q^\beta(s, a) + \lambda \log \frac{\beta(a \mid s)}{\pi(a \mid s)} \right] \label{eq:fqe}
\end{align}
where $Q^\beta(s, a) = \lim_{t\rightarrow \infty} Q_t(s, a)$ and
\begin{align*}
    & Q_{t+1} \gets \argmin_Q \E_{p(s, a)} \left[\left(Q(s, a) - y^{\beta, Q_t}(s, a) \right)^2 \right] \\
    & y^{\beta, Q_t}(s, a) \triangleq r(s, a) + \gamma \E_{\substack{p(s' \mid s, a)\\\beta(a' \mid s')}}\left[Q_t(s', a')\right].
\end{align*}
The scalar $\lambda$ is the regularization coefficient and $\beta(a \mid s)$ is an estimate of the behavioral policy, typically learned via behavioral cloning. Like most TD methods~\citep{haarnoja2018soft, Mnih2013PlayingAW, fujimoto2018addressing}, the TD targets $y$ are not considered learnable.
In practice, most methods do not solve optimize the critic to convergence at each step, instead taking just a few gradient steps before updating the TD targets.
This one-step critic loss is different from the multi-step critic losses used in other RL methods (e.g., TD3, SVG(0)) because it uses the TD target $y^{\beta, Q}(s, a)$ (corresponds to a fixed policy) rather than $y^{\pi, Q}(s, a)$ (corresponding to a sequence of learned policies).
One-step RL amounts to performing one step of policy iteration, rather than full policy optimization. While truncating the iterations of policy iteration can be suboptimal, it can also be interpreted as a form of early stopping regularization.

\paragraph{Example of critic regularization:~\citet{kumar2020conservative}.}
CQL~\citep{kumar2020conservative} modifies the standard Bellman loss to include an additional term that decreases the values predicted for unseen actions.  The actor objective is to maximize Q values; some CQL implementations also regularize the actor loss~\citep{hoffman2020acme, kumar2020conservative}). The objectives can then be written as
\begin{align}
    & \max_\pi  \E_{p(s)\pi(a \mid s)}\left[ Q^\pi(s, a) \right] \label{eq:cql}
\end{align}
where $Q^\pi(s, a) = \lim_{t\rightarrow \infty} Q_t(s, a)$ and
\begin{align*}
Q_{t+1} = & \argmin_Q \E_{p(s, a)} \left[\left(Q(s, a) - y^{\pi, Q_t}(s, a)\right)^2 \right] \\
& + \lambda \left( \E_{p(s) \pi(a \mid s)}\left[Q(s, a) \right]  - \E_{p(s)\beta(a \mid s)}\left[Q(s, a) \right]\right). 
\end{align*}
The second term decreases the Q-values for unseen actions (those sampled from $\pi(a \mid s)$) while the third term increases the values predicted for seen actions (those sampled from the behavioral policy $\beta(a \mid s)$).
Unlike standard temporal difference methods, the CQL updates resemble a competitive game between the actor and the critic. In practice, this cyclic dependency can create unstable learning~\citep{kumar2020conservative, hoffman2020acme}.

\subsection{How are these methods connected?}

Prior work has observed that one-step methods and critic regularization
methods perform similarly on many~\citep{fujimoto2021minimalist, emmons2021rvs} (but not all~\citep{kostrikov2021offline}) tasks.
Despite the differences in objectives and implementations of these two methods (and, more broadly, the actor/critic regularization methods for which they are prototypes), are there deeper, unifying connections between the methods?

In the next section, we introduce a different actor-critic method that will allow us to draw a connection
between one-step RL and critic regularization. We experimentally validate this equivalence in Sec.~\ref{sec:experiments-exact}.
Despite its difference from practically-used methods, such as one-step RL and CQL, we will show that it makes accurate predictions about the behavior of these practical methods (Sec.~\ref{sec:experiments-tab} and~\ref{sec:experiments-cts}).

\subsection{Classifier Actor Critic}
\label{sec:classifier-ac}

To support our analysis, we will introduce a new actor-critic algorithm. This algorithm is similar to prior work, but trains the critic using a cross entropy loss instead of an MSE loss. We introduce this algorithm not because we expect it to perform better than existing actor-critic methods, but rather because it allows us to make precise a connection between actor and critic regularization.  This method treats the value function like a classifier, so we will call it \emph{classifier actor critic}. We will then introduce actor-regularized and critic-regularized versions of this method. The subsequent section (Sec.~\ref{sec:analysis}) will show that these two regularized methods learn the same policy.

The key to our analysis will be to treat Q-values like probabilities, so we define the critic loss in terms of a cross-entropy loss, similar to prior work~\citep{kalashnikov2018scalable, eysenbach2021replacing}. Recalling that Q-values are positive (Sec.~\ref{sec:notation}), we transform the Q-values to have the correct range by using $\frac{Q}{Q+1} \in [0, 1)$. We will minimize the cross-entropy loss applied to the transformed Q-values:
\begin{align*}
    & \E_{\substack{p(s, a)}}\left[ \mathcal{CE}\left(\frac{Q(s, a)}{Q(s, a) + 1}; \frac{y^{\pi, Q_t}(s, a)}{y^{\pi, Q_t}(s, a) + 1} \right) \right] \\
    &= -\E_{p(s, a)}\left[\frac{y^{\pi, Q_t}(s, a)}{y^{\pi, Q_t}(s, a) + 1} \log \frac{Q(s, a)}{Q(s, a) + 1} \right. \nonumber \\
    & \hspace{8em} \left. + \frac{1}{y^{\pi, Q_t}(s, a) + 1} \log \frac{1}{Q(s, a) + 1} \right] \nonumber \\
\end{align*}
\begin{align}
    & \stackrel{\text{const.}}{=} -\E_{p(s, a)} \bigg[y^{\pi, Q_t}(s, a) \log \frac{Q(s, a)}{Q(s, a) + 1} \nonumber \\
    & \hspace{4em} \underbrace{\hspace{8em}  + \log \frac{1}{Q(s, a) + 1} }_{\triangleq \gL_\text{critic}(Q, y^{\pi, Q_t})}\bigg]. \label{eq:critic}
\end{align}
In the last line we scale both the positive and negative term by $y^{\pi, Q_t}(s, a) + 1$, a choice that does not change the optimal classifier but reduces notational clutter.
When the TD target can be computed exactly, solving this optimization problem results in performing one SARSA update: $Q(s, a) \gets r(s, a) + \gamma Q(s', a')$ (see Lemma~\ref{lemma:1}). Thus, by solving this optimization problem many times, each time using the previous Q-value to compute the TD targets, we will converge to the correct Q-values (see Lemma~\ref{lemma:1}).
The actor objective is to maximize the expected \emph{log} of the Q-values: \looseness=-1
\begin{align}
    & \max_\pi \gL_\text{actor}(\pi) \triangleq \E_{p(s) \pi(a \mid s)}\left[\log(Q^\pi(s, a)) \right] \label{eq:actor}
\end{align}
where $Q^\pi(s, a) = \lim_{t\rightarrow \infty} Q_t(s, a)$ and
\begin{align*}
    Q_{t+1} = \argmin_Q \gL_\text{critic}(Q, y^{\pi, Q_t}).
\end{align*}
While most actor-critic methods do not use the logarithm transformation, prior work on conditional behavioral cloning (e.g.,~\citep{savinov2018semi, ding2019goal, sun2019policy, ghosh2020learning, srivastava2019training}) implicitly includes this transformation~\citep{eysenbach2022imitating}.
In the absence of additional regularization, the optimal policy $\pi(a \mid s) = \mathbbm{1}(a = \argmax_{a'} Q(s, a'))$ 
is the same as the optimal policy for the standard actor objective (without the logarithm).
We next introduce a one-step version of this method, as well as a critic regularization variant that resembles CQL. While we will implicitly use a regularization coefficient of $1$ below, Appendix~\ref{appendix:lambda} discusses versions of classifier actor critic with varying degrees of regularization.

\paragraph{One-step RL.}
To make classifier actor critic resemble one-step RL~\citep{brandfonbrener2021offline}, we make two changes: estimating the value of the behavioral policy and adding a regularization term to the actor objective. To estimate the value of the behavioral policy, we modify the critic loss to sample the next action $a'$ from the behavioral policy (i.e., we use $y^{\beta, Q_t}(s, a)$ rather than $y^{\pi, Q_t}(s, a)$).
We also regularize the policy by adding a relative entropy term to the actor loss, analogous to the reverse KL penalty used in one-step RL:
\begin{align}
    & \max_\pi \E_{p(s)\pi(a \mid s)}\left[\log Q^\beta(s, a) + \log \beta(a \mid s) - \log \pi(a \mid s) \right] \label{eq:policy-regularized}
    \end{align}
where $Q^\beta(s, a) = \lim_{t \rightarrow \infty} Q_t(s, a)$ and
\begin{align*}
 Q_{t+1} = \argmin_Q \gL_\text{critic}(Q, y^{\beta, Q_t}).
\end{align*}
In tabular settings, this critic objective estimates the Q-values for $\beta(a \mid s)$ (Appendix Lemma~\ref{lemma:1}).

\paragraph{Critic regularization.}
To emulate CQL, we modify the critic loss (Eq.~\ref{eq:critic}) by adding a penalty term that decreases the values for unseen actions. Whereas CQL applies this penalty to the Q-values directly, we will apply it to the logarithm of the Q-values:\footnote{From a dimensional analysis perspective~\citep{huntley1967dimensional}, this choice makes sense because it allows the penalty term to have the same ``units'' as the critic loss: log Q-values. A second motivation for regularizing the logarithm is that the actor loss uses a logarithm.}
\begin{align}
    & \max_\pi \; \E_{p(s)\pi(a \mid s)}\left[\log Q_r^\pi(s, a) \right] \label{eq:critic-regularized}
    \end{align}
where $Q_r^\pi(s, a) = \lim_{t \rightarrow \infty} Q_t(s, a)$ and $Q_{t+1}(s, a) = \argmin_Q \gL_\text{critic}^r(Q, y^{\pi, Q_t})$:
\begin{align*}
   &\gL_\text{critic}^r(Q, y^{\pi, Q_t}) \triangleq \gL_\text{critic}(Q, y^{\pi, Q_t}) \\
   & + \lambda \Big(\E_{\!\substack{p(s)\\\pi(a \mid s)}\!\!}\left[ \log (Q(s, a) + 1)\right] - \E_{\substack{\!p(s)\\\beta(a \mid s)}\!\!}\left[ \log (Q(s, a) + 1) \right] \Big).
\end{align*}

\section{A Connection between One-Step RL and Critic Regularization}
\label{sec:analysis}

This section provides our main result, which is that actor and critic regularization yield the same policy under some settings. The key to proving this connection will be to analyze the Q-values learned by critic regularization. While we mainly focus on the single-task setting, Sec.~\ref{sec:other-settings} describes how similar results also apply to other settings, including goal-conditioned RL, example-based control, and settings with smaller degrees of regularization. All proofs are in Appendix~\ref{appendix:proofs}.

To relate one-step RL to critic regularization, we start by analyzing the Q-values learned by both methods. We first show that the classifier critic converges to the correct Q-values:
\begin{lemma} \label{lemma:1}
Assume that states and actions are tabular (discrete and finite), that rewards are positive, and that TD targets can be computed exactly (without sampling).
Incrementally update the critic by solving a sequence of optimization problems:
\begin{equation*}
    Q_{t+1} \gets \argmin_Q \gL_\text{critic}(Q, y^{\pi, Q_t}).
\end{equation*}
This sequence of Q-functions will converge to $Q^\pi$:
\begin{equation*}
    \lim_{t \rightarrow \infty} Q_t(s, a) = Q^\pi(s, a) \text{ for all states $s$ and actions $a$}.
\end{equation*}
\end{lemma}
Because one-step RL trains the critic using $\gL_\text{critic}(Q, y^{\beta, Q})$, it learns Q-values corresponding to $Q^\beta(s, a)$. When regularization is added to the critic updates, it learns different Q-values. Perhaps surprisingly, this regularization means that our estimates for the value of policy $\pi(a \mid s)$ look like the value of the original behavioral policy:
\begin{lemma} \label{lemma:critic-reg}
Assume that states and actions are tabular (discrete and finite), that rewards are positive, and that TD targets can be computed exactly (without sampling).
Incrementally update the critic by minimizing a sequence of regularized critic losses using policy $\pi$ and hyperparameter $\lambda = 1$:
\begin{equation*}
    Q_{t+1} \gets \argmin_Q \gL_\text{critic}^r(Q, y^{\pi, Q_t}).
\end{equation*}
In the limit, this sequence of Q-functions will converge to the Q-values for the behavioral policy ($\beta(a \mid s)$), weighted by the ratio of the behavioral and online policies:
\begin{equation*}
    \lim_{t \rightarrow \infty} Q_t(s, a) = \frac{Q^{\beta}(s, a)\beta(a \mid s)}{\pi(a \mid s)}
\end{equation*}
for all states $s$ and actions $a$.
\end{lemma}
\begin{hproof}
\vspace{-0.5em}
The ratio $\frac{\beta(a \mid s)}{\pi(a \mid s)}$ above is an importance weight. Ordinarily, a TD backup for policy $\pi(a \mid s)$ would entail sampling an action $a \sim \pi(a \mid s)$. However, this importance weight means that TD backup is effectively performed by sampling an action $a \sim \beta(a \mid s)$. Such a TD backup resembles the TD backup for $\beta(a \mid s)$. The full proof is in Appendix~\ref{appendix:proofs}.
\end{hproof}
\vspace{-0.5em}
Intuitively, this result says that critic regularization reweights the Q-values to assign higher values to in-distribution actions, where $\beta(a \mid s)$ is large. An unexpected part of this result is that the Q-values correspond to the behavioral policy.
In other words, critic regularization added to a multi-step RL method (one using $y^{\pi, Q_t}(s, a)$) yields the same critic as a one-step RL method (one using $y^{\beta, Q_t}(s, a)$).
Our main result is a direct corollary of this Lemma:
\begin{theorem} \label{thm:main}
Let a behavioral policy $\beta(a \mid s)$ be given and let $Q^\beta(s, a)$ be the corresponding value function. Let $\pi(a \mid s)$ be an arbitrary policy (typically learned) with support constrained to $\beta(a \mid s)$ (i.e., $\pi(a \mid s) > 0 \implies \beta(a \mid s) > 0$). Let $Q_r^\pi(s, a)$ be the critic obtained by the regularized critic update (Eq.~\ref{eq:critic-regularized}) to this policy with $\lambda = 1$. Then critic regularization results in the same policy as one-step RL:
{\footnotesize \begin{equation*}
    \E_{\pi(a \mid s)}\left[\log Q_r^\pi(s, a)\right] = \E_{\pi(a \mid s)}\left[\log Q^\beta(s, a) + \log \frac{\beta(a \mid s)}{\log \pi(a \mid s)} \right]
\end{equation*}}
for all states $s$.
\end{theorem}

\begin{figure*}[t]
    \centering
    \begin{subfigure}[t]{0.65\textwidth}
        \centering
        \includegraphics[width=0.4615\linewidth]{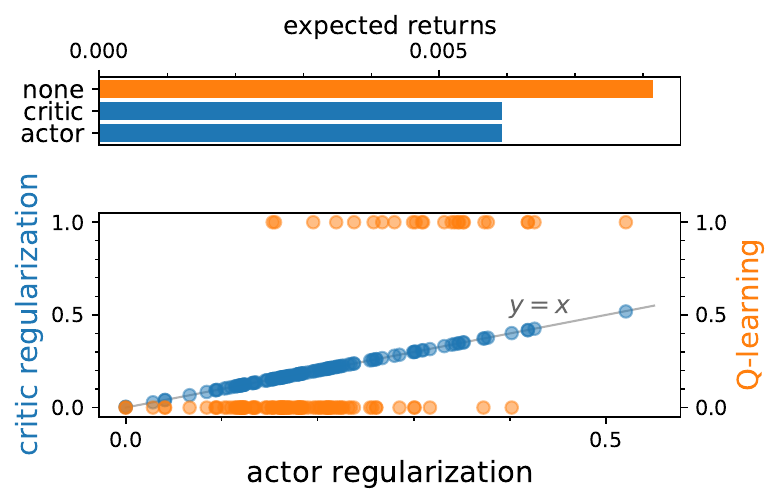}
        \includegraphics[width=0.4615\linewidth]{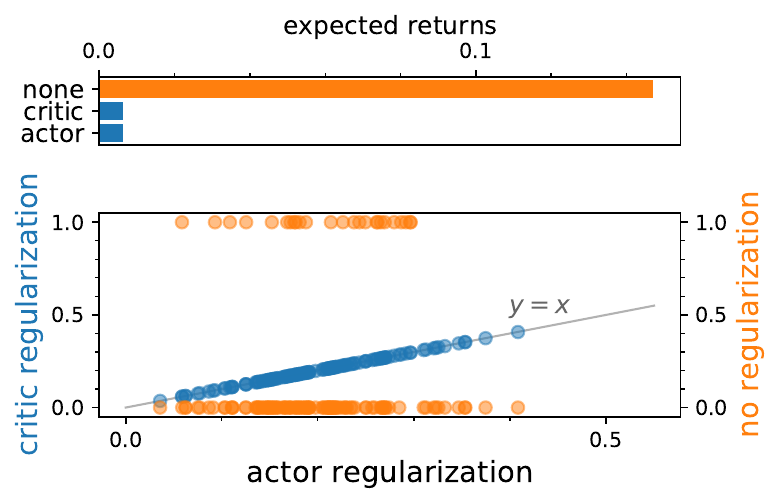}
        \caption*{ MDPs where regularization decreases returns.}
    \end{subfigure}
    \hfill
    \begin{subfigure}[t]{0.3\textwidth}
        \includegraphics[width=\linewidth]{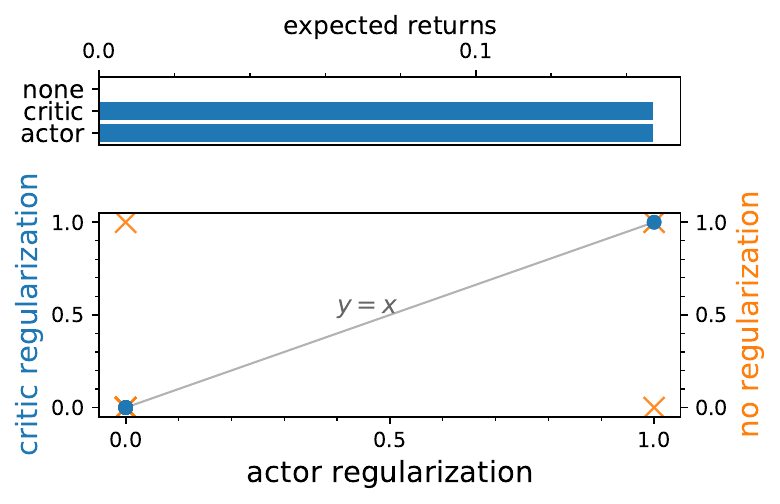}
        \caption*{ MDP where regularization increases returns.}
    \end{subfigure}
    \vspace{-1em}
    \caption{ \textbf{Actor and critic regularization produce identical policies.} Across three tabular settings, we plot the action probabilities $\pi(a \mid s)$ for the policies produced by one-step RL and critic-regularized classifier actor-critic ($R^2 \ge 0.999$). We also plot the action probabilities for a policy learned by an unregularized policy to confirm that the equivalence between one-step RL and critic regularization is not a coincidence.}
    \label{fig:gridworld}
\end{figure*}

\vspace{-0.5em}
Since both forms of regularization result in the same objective for the actor, they must produce the same policy in the end. While prior work has mentioned that critic regularization implicitly regularizes the policy~\citep{yu2021combo}, this result shows that under the assumptions stated above, the implicit regularization of critic regularization results in the exact same policy learning objective as one-step RL. 
This equivalence holds when $\lambda = 1$, and not necessarily for other regularization coefficients. Appendix~\ref{appendix:lambda} shows how a variant of this result that includes an additional regularization mechanism does apply to different regularization coefficients. 
This connection between one step RL and critic regularization concerns their \emph{objective functions}, not the \emph{procedures} used to optimize those objective functions. Indeed, because practical offline RL algorithms sometimes use different optimization procedures (e.g., TD vs. MC estimates of $Q^\beta(s, a)$), they will incur errors in estimating $Q^\beta(s, a)$, violating Theorem~\ref{thm:main}'s assumption that these Q-values are estimated exactly.

\paragraph{Limitations.}

Our theoretical analysis makes assumptions that may not always hold in practice. For example, our results use a critic loss based on the cross entropy loss, while most (but not all~\citep{kalashnikov2018scalable, eysenbach2020c}) practical methods use the MSE. Our analysis assumes that critic regularization arrives at an equilibrium, 
and ignores errors introduced by function approximation and sampling. Nonetheless,
our theoretical results will make accurate predictions about prior offline RL methods. \looseness=-1

\subsection{Extensions of the Analysis}
\label{sec:other-settings}

We extend this analysis in three ways.
\emph{First}, we also show that a similar connection can be established for lesser degrees of regularization ($\lambda < 1$)  (see Appendix~\ref{appendix:lambda}).
\emph{Second}, we show that a similar connection holds for RL problems defined via success examples~\citep{pinto2016supersizing, tung2018reward, kalashnikov2021mt, singh2019end, zolna2020offline, calandra2017feeling, eysenbach2021replacing}
These results use existing actor-critic method, rather than classifier actor critic (see Appendix~\ref{appendix:goals}).
\emph{Third}, we extend our analysis to multi-task settings by looking at goal-conditioned RL problems. \looseness=-1
Taken together, these extensions show that the connection between actor and critic regularization extends to other commonly-studied problem settings.

\section{Numerical Simulations}

Our numerical simulations study whether the theoretical connection between actor and critic regularization holds empirically. The first experiments (Sec.~\ref{sec:experiments-exact}) will use classifier actor-critic, and we will expect the equivalence to hold exactly in this setting. We then study whether this connection still holds for practical prior methods (one-step RL and CQL), which violate our assumptions. We study these commonly-used methods in both tabular settings (Sec.~\ref{sec:experiments-tab}) and on a benchmark offline RL task with continuous states and actions (Sec.~\ref{sec:experiments-cts}).
We do not expect these methods to always be the same (see, e.g.,~\citet[Table 1]{kostrikov2021offline}), and we will focus our experiments on critic regularization with moderate regularization coefficients. See Appendix~\ref{appendix:details} for details and hyperparameters for the experiments. Code for the tabular experiments is available online.\footnote{Code: \tiny \url{https://github.com/ben-eysenbach/ac-connection}}

\subsection{Exact Equivalence for Classifier Actor Critic}
\label{sec:experiments-exact}

\begin{figure*}[t]
    \centering
    \begin{subfigure}[b]{0.19\textwidth}
        \includegraphics[width=\linewidth]{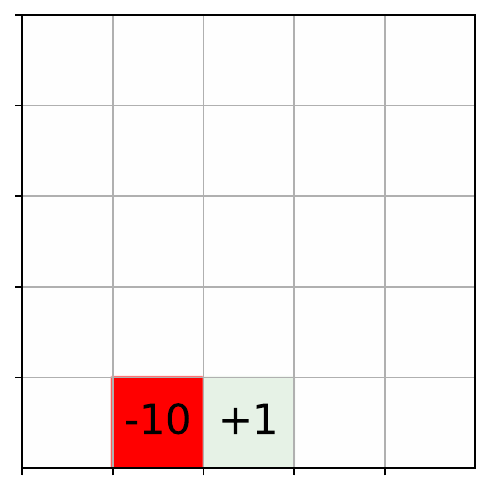}
        \caption{ reward function}
    \end{subfigure}
    \hfill
    \begin{subfigure}[b]{0.19\textwidth}
        \includegraphics[width=\linewidth]{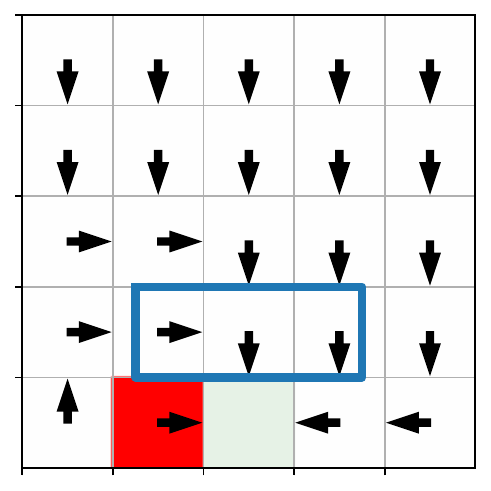}
        \caption{ Q-learning}
    \end{subfigure}
    \hfill
     \begin{subfigure}[b]{0.19\textwidth}
        \includegraphics[width=\linewidth]{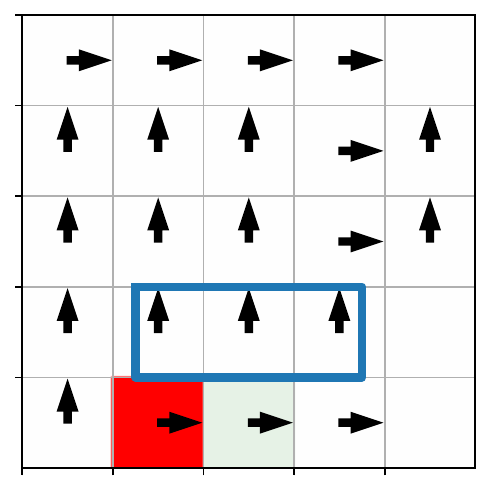}
        \caption{ One-step RL}
    \end{subfigure}
    \hfill
    \begin{subfigure}[b]{0.19\textwidth}
        \includegraphics[width=\linewidth]{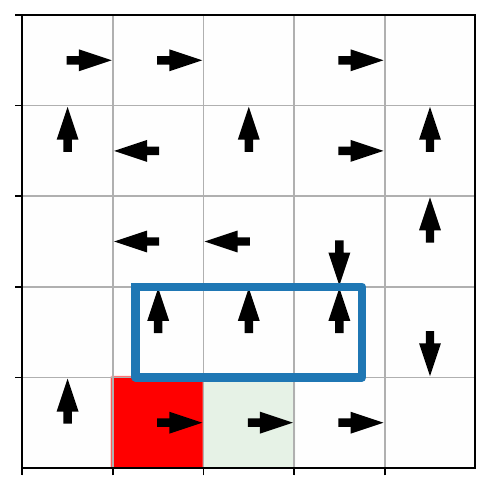}
        \caption{ CQL($\lambda=10$)}
    \end{subfigure}
    \begin{subfigure}[b]{0.19\textwidth}
        \includegraphics[width=\linewidth]{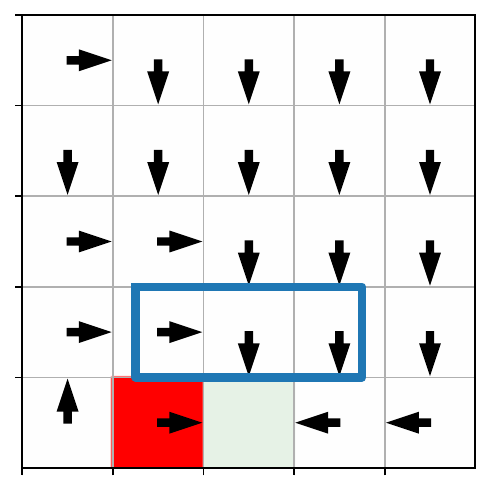}
        \caption{ CQL($\lambda=0.1$)}
    \end{subfigure}
    \vspace{-0.5em}
    \caption{ \textbf{CQL can behave like one-step RL.}
    We design a gridworld \emph{(a)} so that one-step RL \emph{(c)} learns a suboptimal policy. For the three cells highlighted in blue, the optimal policy \emph{(b)} navigates towards the high-reward state (green) while the one-step RL policy \emph{(c)} navigates away from the high-reward state. \emph{(d)} CQL with a large regularization coefficient exhibits the same suboptimal behavior as one-step RL, taking actions that lead away from the high-reward states. \emph{(e)} CQL with a small regularization coefficient behaves like Q-learning.
    For clarity, we only show the argmax action in each state; we omit the arrow when the argmax action is ``do nothing''.}
    \label{fig:cql-sarsa}
\end{figure*}

Our first experiment aims to validate our theoretical result under the required assumptions: when using classifier actor-critic as the RL algorithm, and when using a tabular environment. We use a $5 \times 5$ deterministic gridworld with 5 actions (up/down/left/right/nothing). We describe the reward function and other details in Appendix~\ref{appendix:details}.
To ensure that critic regularization converges to a fixed point and to avoid oscillatory learning dynamics, we update the policy using an exponential moving average.  We also include (unregularized) classifier actor-critic to confirm that regularization is important in some settings. \looseness=-1

We compare these three methods in three environments. The first setting (Fig.~\ref{fig:gridworld} \emph{(left)}) checks our theory that one-step RL and critic regularization should obtain the same policy. 
The second setting (Fig.~\ref{fig:gridworld} \emph{(center)}) shows that one-step RL and critic regularization learn the same (suboptimal) policy in settings where using the Q-values for the behavioral policy lead to a suboptimal policy. 
The final setting is designed so that regularization increases the expected returns.
The dataset is a single trajectory from the initial state to the goal. With such limited data, unregularized classifier actor critic overestimates the Q-values at unseen actions, learning a policy that mistakenly takes these actions. In contrast, the regularized approaches learn to imitate the expert trajectory. Fig.~\ref{fig:gridworld} \emph{(right)} shows that both forms of regularization produce the optimal policy.
In summary, these tabular experiments validate our theoretical results, including in settings where regularization is useful and harmful. These experiments also demonstrate that the actor-critic method introduced in Sec.~\ref{sec:classifier-ac} does converge (Lemma~\ref{lemma:1}).

\subsection{Predictions about Prior Methods: Tabular Setting}
\label{sec:experiments-tab}

Based on our theoretical analysis, we predict that practical implementations of one-step RL and critic regularization will exhibit similar behavior, for a certain critic regularization coefficient. 
This section studies the tabular setting, and the following section will use a continuous control benchmark. For critic regularization, we used CQL~\citep{kumar2020conservative} together with soft value iteration; following~\citep{brandfonbrener2021offline}, we implement one-step RL (reverse KL) using Q-learning.

We designed a deterministic gridworld so one-step RL would fail to learn the optimal policy (see Fig.~\ref{fig:cql-sarsa} \emph{(left)}). If CQL interpolates between the behavioral policy (random) and the optimal policy, then the argmax action would always be the same as the action for $\pi^*$. Based on our analysis, we make a different prediction: that CQL will learn a policy similar to the one-step RL policy.
We show results in Fig.~\ref{fig:cql-sarsa}, just showing the argmax action for visual clarity. The CQL policy takes actions away from both the high-reward state and the low reward state, similar to the behavioral policy but different from both the behavioral policy and the optimal policy. This experiment suggests that CQL can exhibit behavior similar to one-step RL. 
Of course, this effect is mediated by the regularization strength: a larger regularization coefficient would cause CQL to learn a random policy, and a coefficient of 0 would make CQL identical to Q-learning.
We extend Fig.~\ref{fig:cql-sarsa} to include classifier actor critic and regularized variants in Appendix Fig.~\ref{fig:cac}.

\begin{figure}[t]
    \centering
    \includegraphics[width=\linewidth]{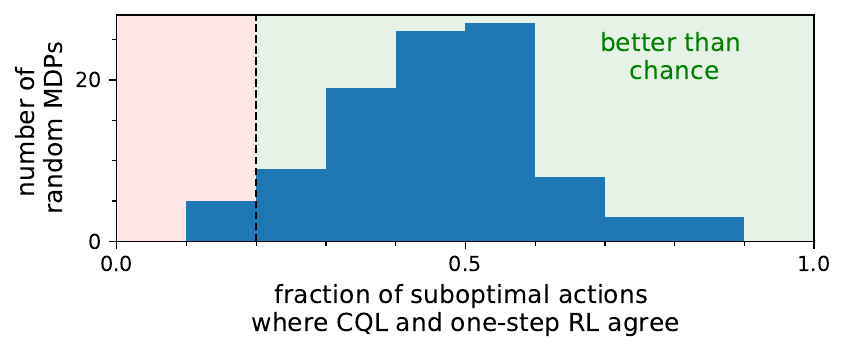}
    \vspace{-2em}
    \caption{ CQL and one-step RL take similar actions on most MDPs that resemble Fig.~\ref{fig:cql-sarsa}}
    \label{fig:histogram}
\end{figure}

\paragraph{How often does one-step RL approximate CQL?}

To show that the results in Fig.~\ref{fig:cql-sarsa} are not cherry-picked, we repeated this experiment using 100 MDPs that are structurally similar to that in Fig.~\ref{fig:cql-sarsa}, but where the locations of the high-reward and low reward state are randomized. In each randomly generated MDP, we determine whether CQL exhibits behavior similar to one-step RL by looking at the states where CQL takes actions that differ from the reward-maximizing actions (as determined by running Q-learning with unlimited data).
Since there are five total actions, a random policy would have a similarity score of 20\%. 
As shown in Fig.~\ref{fig:histogram}, the similarity score is significantly higher than chance for the vast majority of MDPs, showing that one-step RL and CQL($\lambda=10$) produce similar policies on most such gridworlds.

\paragraph{When does one-step RL approximate CQL?}
Because one-step RL is highly regularized (policy iteration is truncated after just one step), one might imagine that it would be most similar to CQL with a very large regularization coefficient. To study this, we use the same environment (Fig.~\ref{fig:cql-sarsa}) and measure the fraction of states where one-step RL and CQL choose the same argmax action. As shown in Fig.~\ref{fig:when}, one-step RL is most similar to CQL with \emph{moderate} regularization ($\lambda = 10$), and is less similar to CQL with a very strong regularization.

\begin{figure}[t]
    \centering
    \includegraphics[width=\linewidth]{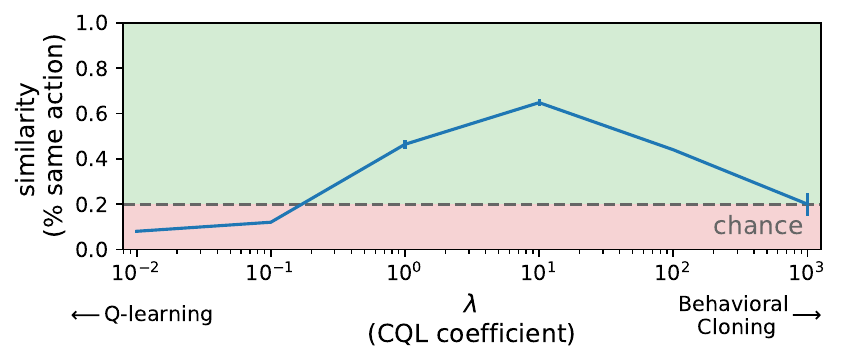}
    \vspace{-2em}
    \caption{ One-step RL is most similar to CQL with a moderate regularization coefficient.}
    \label{fig:when}
\end{figure}

\subsection{Predictions about Prior Methods: Continuous Control Setting}
\label{sec:experiments-cts}

Our final set of experiments studies whether our theoretical results can make accurate testable predictions about practically-used regularization methods in a setting where they are commonly used: offline RL benchmarks with continuous states and actions.
For these experiments, we will use well-tuned implementations of CQL and one-step RL from~\citet{hoffman2020acme}, using the default hyperparameters without modification. 
While our theoretical results do not apply directly to these practical methods, which violate the assumptions in our analysis, they nonetheless raise the questions of whether these methods perform similarly in practice.
We made one change to the one-step RL implementation to makethe comparison more fair:  because CQL learns two Q functions and takes the minimum (a trick introduced in~\citet{fujimoto2018addressing}), we applied this same parametrization to the one-step RL implementation.
Since offline RL methods can perform differently on datasets of varying quality~\citep{wang2020critic, fujimoto2021minimalist, paine2020hyperparameter, wang2021instabilities, fujimoto2019off}, we will repeat our experiments on four datasets from the D4RL benchmark~\citep{fu2020d4rl}. %

\paragraph{Lower bounds on Q-values.}
One oft-cited benefit of critic regularization is that it has guarantees about value-estimation~\citep{kumar2020conservative}: 
under appropriate assumptions, the learned value function will underestimate the discounted expected returns of the policy. Because our analysis shows a connection between one-step RL and critic regularization, it raises the question of whether one-step RL methods have similar value-estimation properties. Taken at face value, this hypothesis seems obvious: the behavioral critic estimates the value of the behavioral policy, so it should underestimate the value of any policy that is better than the behavioral policy. Despite this, the lower bound property of methods like one-step RL are rarely discussed, suggesting that it has yet to be widely appreciated.

\begin{figure}[t]
    \centering
    \includegraphics[width=\linewidth]{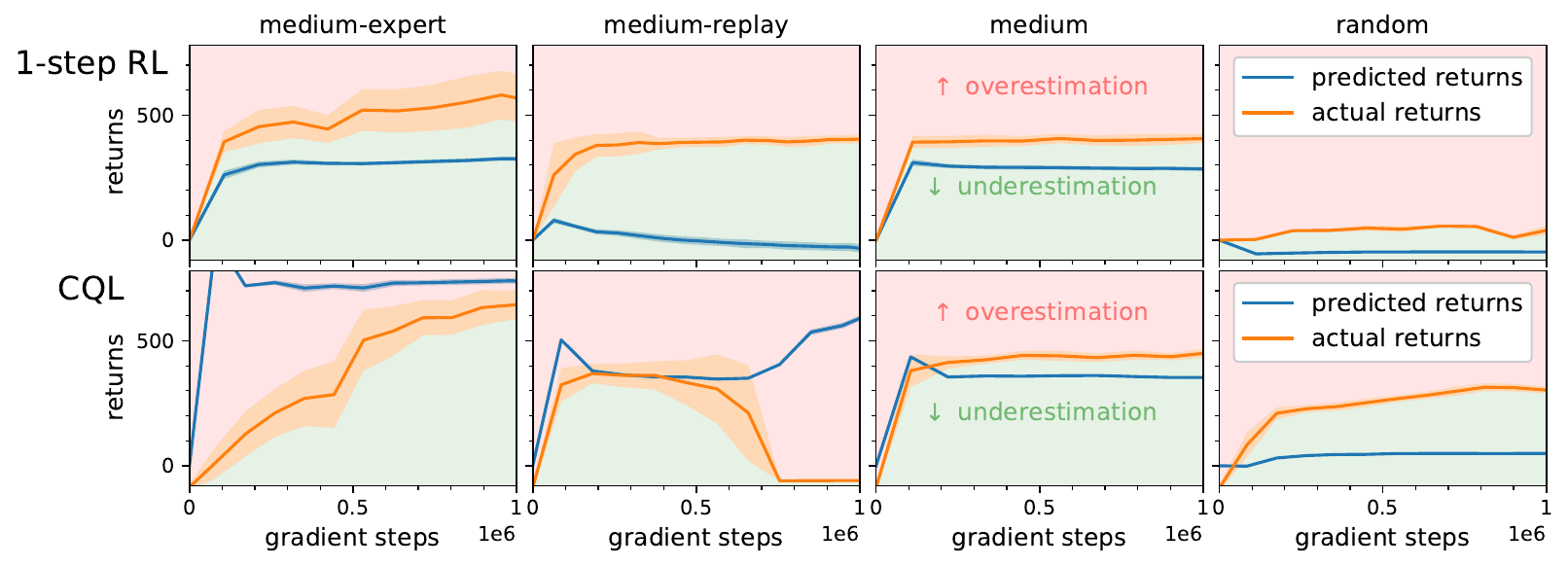}
    \caption{ \textbf{Q-value under/over-estimation.} \figtop \, Experiments on benchmark datasets of varying quality show that one-step RL underestimates the Q-values. \figbottom \, Despite the theoretical guarantees about critic regularization (CQL) yielding underestimates, in practice we observe that the values learned via critic regularization can sometimes overestimate the actual returns. We plot the mean and standard deviation across five random seeds.
    Note that the Q-values are equivalent to the value function $V^\pi(s)$.}
    \label{fig:q_vals}
\end{figure}

Fig.~\ref{fig:q_vals} shows both these predicted and actual (discounted) returns throughout the course of training.
The results for one-step RL confirm our theoretical prediction on $\nicefrac{4}{4}$ datasets: the Q-values from one-step RL underestimate the actual returns.
In contrast, we observe that critic regularization overestimates the true returns on $\nicefrac{2}{4}$ environments, perhaps because the regularization coefficients used to achieve good returns in practice are too weak to guarantee the lower bound property, and perhaps because the theoretical guarantees are only guaranteed to hold at convergence.
In total, these experiments confirm our theoretical predictions that one-step RL will result in Q-values that are underestimations, while also questioning the claim that critic regularization methods are always preferable for ensuring underestimation.

\paragraph{Critic regularization causes actor regularization.}
Our analysis in Sec.~\ref{sec:analysis} not only suggests that one-step RL methods might inherit properties of critic regularization (as studied in the previous section), but also suggests that critic regularization methods may behave like one-step methods. In particular, while critic regularization methods such as CQL do not explicitly regularize their actor, we hypothesize that they \emph{implicitly} regularize the actor (Lemma~\ref{lemma:critic-reg}), similar to how one-step RL methods explicitly regularize the actor.

\begin{figure}[t]
    \centering
    \includegraphics[width=\linewidth]{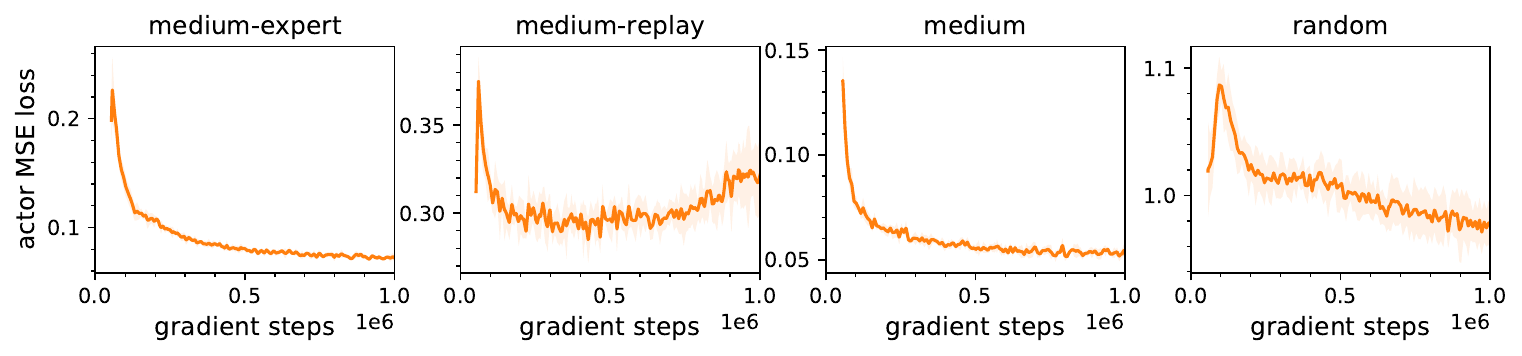}%
    \caption{ \textbf{Critic regularization causes actor regularization.} Performing critic regularization via CQL implicitly results in actor regularization, similar to one-step RL: the MSE between the predicted actions and the dataset actions decreases. We plot the mean and standard deviation across five random seeds.}
    \label{fig:mse-loss}
\end{figure}

We measure the MSE between the action in the dataset and the action predicted by the learned policy. Fig.~\ref{fig:mse-loss} shows the results. Some CQL implementations, including ours, ``warm-start'' the actor by applying a behavioral cloning loss for 50,000 iterations; we omit these initial gradient steps from our plots so that any effect is caused solely by the critic regularization. On \nicefrac{4}{4} datasets, we observe that the MSE between the CQL policy's actions and the actions in the datasets decreases throughout training. Perhaps the one exception is on the \texttt{medium-replay} dataset, where the MSE eventually starts to increase after 5e5 gradient steps. While directly regularizing the actor leads to MSE errors that are $\sim3\times$ smaller, this plot nevertheless provides evidence that critic regularization indirectly regularizes the actor. \looseness=-1

\section{Conclusion}

In this paper, we drew a connection between two seemingly-distinct RL regularization methods: one-step RL and critic regularization. While our analysis made assumptions that are typically violated in practice, it nonetheless made accurate, testable predictions about practical methods with commonly-used hyperparameters: critic regularization methods can behave like one-step methods, and vice versa. \looseness=-1

{
\paragraph{Acknowledgements.} We thank Aviral Kumar, George Tucker, David Brandfonbrener, and Scott Fujimoto for discussions and feedback on drafts of the paper. We thank Young Geng and Ilya Kostrikov for sharing code. We thank Chongyi Zheng for spotting a bug that prompted this project. This material is supported by the Fannie and John Hertz Foundation, NSF GRFP (DGE1745016), ONR N000142312368, DARPA/AFRL FA87502321015, and the Office of Naval Research

}

\clearpage
\appendix
\onecolumn

\section*{Appendices}

In the Appendices, we provide proofs of the theoretical results (Appendix~\ref{appendix:goals}), extend the analysis to other RL settings (Appendices~\ref{appendix:goals}--\ref{appendix:examples}), and then provide details of the experiments (Appendix~\ref{appendix:details}).

\section{Proofs}
\label{appendix:proofs}

\subsection{Proof of Lemma~\ref{lemma:1}}

\emph{Proof sketch.} Lemma~\ref{lemma:1} shows that classifier actor critic converges. The key idea of the proof will be to show that the incremental updates for classifier actor critic are exactly the same as the incremental updates for Q-learning. Q-learning converges, so an algorithm that performs the same incremental updates as Q-learning must also converge.

\begin{proof}
As the cross entropy loss is minimized when the predictions equal the labels, updates for $\gL_\text{critic}(Q, \pi)$ can be written as $\frac{Q(s, a)}{Q(s, a) + 1} \gets \frac{y^{\pi, Q_t}(s, a)}{y^{\pi, Q_t}(s, a) + 1}$. If the updates are performed by averaging over all possible next states (e.g., in the tabular setting), these updates are equivalent to directly updating $Q(s, a) \gets y^{\pi, Q_t}(s, a) = r(s, a) + \gamma \E_{p(s' \mid s, a)\pi(a' \mid s')}\left[Q_t(s', a') \right]$, which is the standard policy evaluation update for policy $\pi(a \mid s)$.
Thus, we can invoke the standard result that policy evaluation converges to $Q^\pi$~\citep[Theorem 1.14.]{agarwal2019reinforcement} to argue that updates for $\gL_\text{critic}$ likewise converge to $Q^\pi$.
\end{proof}

In this proof, the TD targets were the expectation over the next state and next action. If Eq.~\ref{eq:critic} were optimized using a single-sample estimate of this expectation, $\rvy = r(s, a) + \gamma Q_t(s', a')$, then the updates would be biased:
\begin{equation*}
\frac{Q(s, a)}{Q(s, a) + 1} \gets \E\left[\frac{\rvy}{\rvy + 1}\right] \le \frac{\E[\rvy]}{\E[\rvy] + 1} = \frac{y^{\pi, Q_t}(s, a)}{y^{\pi, Q_t}(s, a) + 1}.
\end{equation*}
In settings with stochastic transitions or policies, these updates would result in estimating a lower bound on $Q^\pi(s, a)$.

\subsection{Proof of Lemma~\ref{lemma:critic-reg} and Theorem~\ref{thm:main}}

\begin{proof}
Our proof proceeds in three steps. First, we derive the update equations for the regularized critic update. That is, if we maintained a table of Q-values, what would the new value for $Q(s, a)$ be? Second, we show that these updates are equivalent to performing policy evaluation on a \emph{re-parametrized} critic $\tilde{Q}(s, a) = Q(s, a) \frac{\pi(a \mid s)}{\beta(a \mid s)}$. We invoke the standard results for policy evaluation to prove convergence that $\tilde{Q}(s, a)$ convergences. Finally, we undo the reparametrization to obtain convergence results for $Q(s, a)$.

\textbf{Step 0.}
We start by rearranging the regularized critic objective:
\begin{align*}
    \gL_\text{critic}^r(Q, y^{\pi, Q_t})& \triangleq \gL_\text{critic}(Q, y^{\pi, Q_t}) + \bigg(\E_{p(s)\pi(a \mid s)}\left[ \log (Q(s, a) + 1)\right] - \E_{p(s)\beta(a \mid s)}\left[ \log (Q(s, a) + 1) \right] \bigg) \\
    &= -\E_{p(s, a)}\left[y^{\pi, Q_t}(s, a) \log \frac{Q(s, a)}{Q(s, a) + 1} + \log \frac{1}{Q(s, a) + 1} \right] \\
    & \qquad + \bigg(\E_{p(s)\pi(a \mid s)}\left[ \log (Q(s, a) + 1)\right] - \E_{p(s)\beta(a \mid s)}\left[ \log (Q(s, a) + 1) \right] \bigg) \\
    &= -\E_{p(s, a)}\left[y^{\pi, Q_t}(s, a) \log \frac{Q(s, a)}{Q(s, a) + 1} + \cancel{\log \frac{1}{Q(s, a) + 1}} \right] \\
    & \qquad - \bigg(\cancel{\E_{p(s)\pi(a \mid s)}\left[ \log \frac{1}{Q(s, a) + 1}\right]} + \E_{p(s)\beta(a \mid s)}\left[ \log \frac{1}{Q(s, a) + 1} \right] \bigg) \\
    &= -\E_{p(s, a)}\left[y^{\pi, Q_t}(s, a) \log \frac{Q(s, a)}{Q(s, a) + 1}\right] + \E_{p(s)\beta(a \mid s)}\left[ \log \frac{1}{Q(s, a) + 1} \right].
\end{align*}
For the cancellation on the third line, we used the fact that $p(s, a) = p(s)\beta(a \mid s)$.

\textbf{Step 1.} To start, note that the regularized critic update is equivalent to a weighted classification loss: positive examples are sampled $(s, a) \sim p(s)\beta(a \mid s)$ and receive weight $\frac{y^{\pi, Q_t}(s, a)}{y^{\pi, Q_t}(s, a) + 1}$, and negative examples are sampled $(s, a) \sim p(s)\pi(a \mid s)$ and receive weight $\frac{1}{y^{\pi, Q_t}(s, a) + 1}$. The Bayes' optimal classifier is given by
\begin{align*}
    \frac{Q(s, a)}{Q(s, a) + 1} = \frac{\frac{y^{\pi, Q_t}(s, a)}{y^{\pi, Q_t}(s, a) + 1}p(s)\beta(a \mid s)}{\frac{y^{\pi, Q_t}(s, a)}{y^{\pi, Q_t}(s, a) + 1}p(s)\beta(a \mid s) + \frac{1}{y^{\pi, Q_t}(s, a) + 1} p(s) \pi(a \mid s)} = \frac{y^{\pi, Q_t}(s, a) \beta(a \mid s)}{y^{\pi, Q_t}(s, a) \beta(a \mid s) + \pi(a \mid s)}.
\end{align*}
Solving for $Q(s, a)$ on the left hand side, the optimal value for $Q(s, a)$ is given by
\begin{equation}
    Q(s, a) = y^{\pi, Q_t}(s, a)\frac{\beta(a \mid s)}{\pi(a \mid s)} = (r(s, a) + \E_{p(s' \mid s, a) \pi(a' \mid s')}[Q_t(s', a')]) \frac{\beta(a \mid s)}{\pi(a \mid s)}. \label{eq:critic-reg-identity}
\end{equation}
This equation tells us what each update for the regularized critic loss does.

\textbf{Step 2.} To analyze these updates, we define $\tilde{Q}(s, a) \triangleq Q_t(s, a) \frac{\pi(a \mid s)}{\beta(a \mid s)}$. Then these updates can be written using $\tilde{Q}(s, a)$ as
\begin{equation}
    \tilde{Q}(s, a)\frac{\beta(a \mid s)}{\pi(a \mid s)} = \left(r(s, a) + \E_{p(s' \mid s, a) \pi(a' \mid s')}\left[\tilde{Q}(s', a')\frac{\beta(a' \mid s')}{\pi(a' \mid s')} \right] \right)\frac{\beta(a \mid s)}{\pi(a \mid s)},
\end{equation}
which can be simplified to
\begin{equation}
    \tilde{Q}(s, a) = r(s, a) + \E_{p(s' \mid s, a) \beta(a' \mid s')}\left[\tilde{Q}(s', a')\right].
\end{equation}
Note that the ratio $\frac{\beta(a' \mid s')}{\pi(a' \mid s')}$ inside the expectation acts like an importance weight, so that the expectation over $\pi(a' \mid s')$ becomes an expectation over $\beta(a' \mid s')$.
Thus, the regularized critic updates are equivalent to perform policy evaluation on $\tilde{Q}(s, a)$. An immediately consequence is that the regularized critic updates converge, and they converge to $\tilde{Q}^*(s, a) = Q^\beta(s, a)$.

\textbf{Step 3.} Finally, we translate these convergence results for $\tilde{Q}(s, a)$ into convergence results for $Q(s, a)$. Written in terms of the original Q-values, we see that the optimal critic for the regularized critic update is
\begin{equation}
    Q^*(s, a) = \tilde{Q}^*(s, a) \frac{\beta(a \mid s)}{\pi(a \mid s)} = Q^\beta(s, a) \frac{\beta(a \mid s)}{\pi(a \mid s)}.
\end{equation}
This completes the proof of Lemma~\ref{lemma:critic-reg}.

\end{proof}
We now prove Theorem~\ref{thm:main} by applying a logarithm:
\begin{proof}
\begin{equation*}
\log Q^*(s, a) = \log \left( Q^\beta(s, a) \frac{\beta(a \mid s)}{\pi(a \mid s)} \right) = \log Q^\beta(s, a) + \log \beta(a \mid s) - \log \pi(a \mid s).
\end{equation*}

\end{proof}

We note that our proof does not account for stochastic and function approximation errors. However, if we assume that the TD updates are deterministic (e.g., as they are in deterministic MDPs), then the updates for classifier actor-critic are identical to those of Q-learning (Lemma~\ref{lemma:1}). Thus, it immediately inherits any theoretical results regarding the propagation of errors for Q-learning.

While this Theorem 4.3 shows that one-step RL and critic regularization have the same fixed point, it does not say how many transitions or gradient updates are required to reach those fixed points.

\subsection{Why use the cross-entropy loss?}

Our proof of Theorem~\ref{thm:main} helps explain why classifier actor-critic use the cross entropy loss for the critic loss, rather than the MSE loss. Precisely, our analysis requires that the optimal Q function be a \emph{ratio}, $\tilde{Q}(s, a) = \frac{Q(s, a) \pi(a \mid s)}{\beta(a \mid s)}$.
The cross entropy loss can readily estimate ratios. For example, the optimal classifier for data drawn from $p(x)$ and $q(x)$ is $C(x) = \frac{p(x)}{p(x) + q(x)}$, so the ratio can be expressed as $\frac{C(x)}{1 - C(x)} = \frac{p(x)}{q(x)}$.
However, fitting a function $C(x)$ to data drawn from (say) a 1:1 mixture of $p(x)$ and $q(x)$ would result in $C(x) = \frac{1}{2}p(x) + \frac{1}{2}q(x)$, which we cannot transform to express the ratio $\frac{p(x)}{q(x)}$ as a function of $C(x)$.

\begin{wrapfigure}[13]{R}{0.5\textwidth}
\centering
\vspace{-2em}
\includegraphics[width=0.9\linewidth]{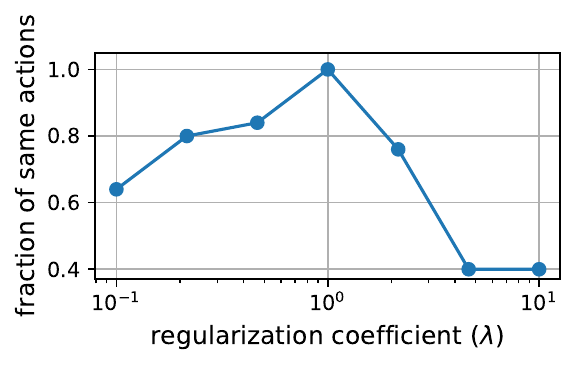}
\vspace{-1em}
\caption{{Under the assumptions of Theorem~\ref{thm:main}, one-step RL is most similar to critic regularization with a coefficient of $\lambda = 1$.}}\label{fig:tab-similarity}
\end{wrapfigure}

\subsection{Validating the Theory}

Our theoretical results suggest that one-step RL and critic regularization should be most similar with critic regularization is applied with a regularization coefficient of $\lambda = 1$. To test this hypothesis, we took the task from Fig.~\ref{fig:gridworld} \emph{(Left)} and measured the similarity between one-step RL and critic-regularized classifier actor critic, for varying values of the critic regularization parameter. We measured the similarity of the policies obtained by the two methods by counting the fraction of states where the two methods choose the same (argmax) action. The results, shown in Fig.~\ref{fig:tab-similarity}, validate our theoretical prediction that these methods should be most similar with $\lambda = 1$.

\subsection{What about using the policy gradient?}
Our analysis fundamentally requires using TD learning: the key step is that doing TD backups using one policy is equivalent to doing (modified) TD backups with a different policy. However, the actor updates for both methods could be implemented using a policy gradient or natural gradient, rather than a straight-through gradient estimator. Indeed, much of the work on one-step RL methods~\citep{peng2019advantage, siegel2020keep} uses an actor update that resembles a policy gradient or natural policy gradient (e.g., 1-step RL with a reverse KL penalty~\citep{brandfonbrener2021offline}).

\section{Varying the regularization coefficient}
While our main analysis (Theorem~\ref{thm:main})showed that regularization and critic regularization yield the same policy when these regularizers are applied with a certain strength, in practice the strength of regularization is controlled by a hyperparameter. This hyperparameter raises a question: \emph{does the connection between one-step RL and critic regularization hold for different values of this hyperparameter?}

In this section, we show that there remains a precise connection between actor and critic regularization, even for different values of this hyperparameter. This result not only suggests that the connection is stronger than initially suggested by the main result. Proving this connection also helps highlight how many regularization methods can be cast from a similar mold.

\subsection{A Regularization Coefficient.}
\label{appendix:lambda}

We start by modifying the actor regularizer and critic regularizer introduced in Sec.~\ref{sec:classifier-ac} to include an additional hyperparameter.

\paragraph{Mixture policy.}
Both the actor and critic losses will make use of a mixture policy, $(1 - \lambda) \pi(a \mid s) + \lambda \beta(a \mid s)$, where $\lambda \in [0, 1]$ will be a hyperparameter. Larger values of $\lambda$ yield a mixture policy that is closer to the behavioral policy; this will correspond to higher degrees of regularization. Mixtures of policies are commonly used in practice~\citep[Appendix F]{kumar2020conservative},\citep[Eq.~11]{villaflor2020fine},~\citep[Sec.~4.3]{finn2016guided}~\citep{lyu2022efficient}~\citep[Eq.~2.5]{hazan2019provably}, even though it rarely appears in theoretical offline RL literature. Indeed, because critic regularization resembles a two-player zero-sum game, mixture policies might even be \emph{required} to find a (Nash) equilibrium of the critic regularizer~\citep{nash1951non}.

\paragraph{$\lambda$-weighted critic loss.}
With this concept of a mixture policy, we define the $\lambda$-weighted actor and critic regularizers.
For the $\lambda$-weighted critic loss, we will change how the TD targets are computed. Instead of sampling the next action from $\pi$ or $\beta$, we will sample the next action from a $\lambda_\text{TD}$-weighted combination of these two policies, reminiscent of how prior work has regularized the actions sampled for the TD backup~\citep{fujimoto2019off, zhou2020plas}:
\begin{equation*}
    y^{\lambda_\text{TD}} \triangleq y^{(1 - \lambda) \pi + \lambda \beta}(s, a) = r(s, a) + \gamma \E_{\substack{p(s' \mid s, a) \\ (1 - \lambda_\text{TD}) \pi(a \mid s) + \lambda_\text{TD} \beta(a \mid s)}}[Q(s', a')].
\end{equation*}
When introducing one-step RL in Sec.~\ref{sec:classifier-ac}, we used $\lambda_\text{TD} = 1$.

Using this TD target, the $\lambda$-weighted critic loss can now be written as a  combination of the un-regularized objective (Eq.~\ref{eq:critic}) plus the regularized objective (Eq.~\ref{eq:critic-regularized}):
\begin{align}
    \gL_\text{critic}^r(Q, \lambda_\text{critic}) &\triangleq (1 - \lambda_\text{critic}) \left(-\E_{p(s, a)}\left[\frac{y^{\lambda_\text{TD}}(s, a)}{y^{\lambda_\text{TD}}(s, a) + 1} \log \frac{Q(s, a)}{Q(s, a) + 1} + \frac{1}{y^{\lambda_\text{TD}}(s, a) + 1} \log \frac{1}{Q(s, a) + 1} \right] \right) \nonumber \\
    & \qquad + \lambda \left(-\E_{\substack{p(s, a)\\a^- \sim \pi(\cdot \mid s)}} \left[\frac{y^{\lambda_\text{TD}}(s, a)}{y^{\lambda_\text{TD}}(s, a) + 1} \log \frac{Q(s, a)}{Q(s, a) + 1} + \frac{1}{y^{\lambda_\text{TD}}(s, a) + 1} \log \frac{1}{Q(s, a) + 1} \right] \right) \nonumber \\
    &= -\E_{\substack{p(s, a)\\a^- \sim (1 - \lambda_\text{critic})\pi(\cdot \mid s) + \lambda_\text{critic}\beta(\cdot \mid s)}}\left[\frac{y^{\lambda_\text{TD}}(s, a)}{y^{\lambda_\text{TD}}(s, a) + 1} \log \frac{Q(s, a)}{Q(s, a) + 1} + \frac{1}{y^{\lambda_\text{TD}}(s, a) + 1} \log \frac{1}{Q(s, a^-) + 1} \right]. \label{eq:critic-lambda}
\end{align}
The second line rewrites this objective: the first term looks the same as the original ``positive'' term in the critic objective, while the ``negative'' term uses actions sampled from a mixture of the current policy and the behavioral policy. When $\lambda_\text{critic} = 1$, we recover the regularized critic loss introduced in Sec.~\ref{sec:classifier-ac}.

\paragraph{$\lambda$-weighted actor loss.}
Finally, the strength of the actor regularizer can be controlled by changing the reverse KL penalty. While it may seem like changing the reward scale would varying the strength of the actor loss, this is not the case for classifier actor critic because of the $\log(\cdot)$ in the actor loss. Instead, we will relax the reverse KL penalty between the learned policy $\pi(a \mid s)$ and the behavioral policy $\beta(a \mid s)$ so that only the mixture policy only needs to be close to behavioral policy:
\begin{align}
   \gL_\text{actor}^r(\pi, \lambda_\text{KL}) &\triangleq \E_{p(s) \pi(a \mid s)}\left[\log Q(s, a) + \log \beta(a \mid s) - \log \left((1 - \lambda_\text{KL}) \pi(a \mid s) + \lambda_\text{KL} \beta(a \mid s) \right) \right]. \label{eq:actor-lambda}
\end{align}
As indicated on the second line, replacing $\beta(a \mid s)$ with the mixture policy has an effect similar to that of decreasing the weight applied to the KL penalty. The approximation on the second line is determined by the Jensen Gap~\citep{abramovich2016some, gao2017bounds}.
When introducing one-step RL in Sec.~\ref{sec:classifier-ac}, we used $\lambda_\text{KL} = 1$, together with $\lambda_\text{TD} = 1$.

In summary, the strength of the actor and critic regularizers can be controlled through additional hyperparameters ($\lambda_\text{critic}, \lambda_\text{TD}, \lambda_\text{KL}$). Indeed, it is typical for offline RL methods to require many hyperparameters~\citep{brandfonbrener2021offline, lu2021revisiting, paine2020hyperparameter, wu2019behavior}, and performance is sensitive to their settings.
However, the close connection that we have shown between actor and critic regularizers allows us to decrease the number of hyperparameters. 

\subsection{Analysis}

In our main result (Thm.~\ref{thm:main}), we showed that one-stel RL and critic regularization are equivalent when $\lambda_\text{critic} = \lambda_\text{TD} = \lambda_\text{KL} = 1$. This is a large value for the regularization strength, and we now consider what happens for smaller degrees of regularization: is there still a connection between one-step RL and critic regularization?

The following theorem will prove that this is the case. In particular, applying critic regularization with coefficient $\lambda_\text{critic}$ yields the same policy as applying one-step RL with $\lambda_\text{TD} = \lambda_\text{KL} = \lambda_\text{critic}$. That is, there is a very simple recipe for converting the hyperparameters for critic regularization into the hyperparameters for one-step RL.

\begin{theorem} \label{thm:lambda}
Let policy $\pi(a \mid s)$ be given, let $Q^\beta(s, a)$ be the Q-function of the behavioral policy, and let $Q_r^{\lambda_\text{TD}}(s, a, \lambda_\text{critic})$ be the critic obtained by the $\lambda_\text{critic}$-weighted regularized critic update (Eq.~\ref{eq:critic-lambda}) using TD targets $y^{\lambda_\text{TD}}(s, a)
$. If $\lambda_\text{critic} = \lambda_\text{TD} = \lambda_\text{KL}$, then the $\lambda_\text{KL}$-weighted actor loss (Eq.~\ref{eq:actor-lambda}) is equivalent to the un-regularized policy objective using the regularized critic:
\begin{align*}
    & \E_{p(s) \pi(a \mid s)}\left[\log Q(s, a) + \log \beta(a \mid s) - \log \left((1 - \lambda_\text{KL}) \pi(a \mid s) + \lambda_\text{KL} \beta(a \mid s) \right) \right] \\
    &= \E_{\pi(a \mid s)}\left[\log Q_r^{\lambda_\text{TD}}(s, a, \lambda_\text{critic})\right]  \qquad \text{for all states }s.
\end{align*}
\end{theorem}
While we used the cross entropy loss for this result, it turns out that the result also holds for the more standard MSE loss (we omit the proof for brevity). 

\paragraph{Limitations.}
Before presenting the proof in Sec.~\ref{sec:lambda-proof}, we discuss a few limitations of this result. Like the rest of the analysis in this paper, the form of the critic regularizer is different from that often used in practice. Additionally, our analysis assumes ignores many sources of errors (e.g., sampling, function approximation), and assumes that each objective is optimized exactly.

\subsection{Proof of Theorem~\ref{thm:lambda}}
\label{sec:lambda-proof}

\begin{proof}
We start by defining the fixed point of the $\lambda$-weighted regularized critic loss. Like in the single-task setting, this loss resembles a weighted classification problem, so we can write down the Bayes' optimal classifier as 
\begin{align*}
    \frac{Q(s, a)}{Q(s, a) + 1} &= \frac{\frac{y^{\lambda_\text{TD}}(s, a)}{y^{\lambda_\text{TD}}(s, a) + 1}p(s)\beta(a \mid s)}{\frac{y^{\lambda_\text{TD}}(s, a)}{y^{\lambda_\text{TD}}(s, a) + 1}p(s)\beta(a \mid s) + \frac{1}{y^{\lambda_\text{TD}}(s, a) + 1} p(s) ((1 - \lambda_\text{critic})\pi(a \mid s) + \lambda_\text{critic} \beta(a \mid s))} \\
    &= \frac{y^{\lambda_\text{TD}}(s, a) \beta(a \mid s)}{y^{\lambda_\text{TD}}(s, a) \beta(a \mid s) + (1 - \lambda_\text{critic})\pi(a \mid s) + \lambda_\text{critic} \beta(a \mid s)}.
\end{align*}
Solving for $Q(s, a)$ on the left hand side, the optimal value for $Q(s, a)$ is given by
\begin{align}
    Q(s, a) &= y^{\lambda_\text{TD}}(s, a)\frac{\beta(a \mid s)}{(1 - \lambda_\text{critic})\pi(a \mid s) + \lambda_\text{critic} \beta(a \mid s)} \nonumber \\
    &= (r(s, a) + \E_{p(s' \mid s, a), a' \sim (1 - \lambda_\text{TD}), \pi(\cdot \mid s') + \lambda_\text{TD} \beta(\cdot \mid s)}[Q(s', a')]) \frac{\beta(a \mid s)}{(1 - \lambda_\text{critic})\pi(a \mid s) + \lambda_\text{critic} \beta(a \mid s)}. \label{eq:mixture-1}
\end{align}
Note that the next action $a'$ is sampled from a mixture policy defined by $\lambda_\text{TD}$. This equation tells us what each update for the $\lambda$-weighted regularized critic loss does.

To analyze these updates, we define
\begin{equation*}
\tilde{Q}(s, a) \triangleq Q(s, a) \frac{(1 - \lambda_\text{critic})\pi(a \mid s) + \lambda_\text{critic} \beta(a \mid s)}{\beta(a \mid s)}.
\end{equation*}

Like before, the ratio $\frac{\beta(a' \mid s')}{(1 - \lambda_\text{TD})\pi(a' \mid s') + \lambda_\text{TD} \beta(a' \mid s')}$ can act like an importance weight. When $\lambda_\text{TD} = \lambda_\text{critic}$, then this importance weight cancels with the sampling distribution, providing the following identity:
\begin{align*}
    & \E_{p(s' \mid s, a), a' \sim (1 - \lambda_\text{TD}), \pi(\cdot \mid s') + \lambda_\text{TD} \beta(\cdot \mid s)}[Q(s', a')] \\
    &= \E_{p(s' \mid s, a), a' \sim (1 - \lambda_\text{TD}), \pi(\cdot \mid s') + \lambda_\text{TD} \beta(\cdot \mid s)} \left[\tilde{Q}(s, a)\frac{\beta(a \mid s)}{(1 - \lambda_\text{critic})\pi(a \mid s) + \lambda_\text{critic} \beta(a \mid s)} \right] \\
    &= \E_{p(s' \mid s, a), a' \sim \beta(\cdot \mid s')}[\tilde{Q}(s, a)].
\end{align*}
Substituting this identity in Eq.~\ref{eq:mixture-1}, we can write the updates using $\tilde{Q}(s, a)$:
\begin{align*}
    & \tilde{Q}(s, a)\frac{\beta(a \mid s)}{(1 - \lambda_\text{critic})\pi(a \mid s) + \lambda_\text{critic} \beta(a \mid s)} \\
    & = \left(r(s, a) + \E_{p(s' \mid s, a), a' \sim \beta(\cdot \mid s')}[\tilde{Q}(s, a)] \right)\frac{\beta(a \mid s)}{(1 - \lambda_\text{critic})\pi(a \mid s) + \lambda_\text{critic} \beta(a \mid s)},
\end{align*}
which can be simplified to
\begin{align*}
    \tilde{Q}(s, a) = r(s, a) + \E_{p(s' \mid s, a), a' \sim \beta(\cdot \mid s')}[\tilde{Q}(s, a)].
\end{align*}

We then translate these convergence results for $\tilde{Q}(s, a)$ into convergence results for $Q(s, a)$. Written in terms of the original Q-values, we see that the optimal critic for the regularized critic update is
\begin{equation}
    Q^*(s, a) = Q^\beta(s, a) \frac{\beta(a \mid s)}{(1 - \lambda_\text{critic})\pi(a \mid s) + \lambda_\text{critic} \beta(a \mid s)}.
\end{equation}
Note that this holds for any value of $\lambda_\text{critic} = \lambda_\text{TD} \in [0, 1]$. This result suggests that two common forms of regularization, decreasing the values predicted at unseen actions and regularizing the actions used in the TD backup, can produce the same effect: a critic that estimates the Q-values of the behavioral policy (multiplied by some importance weight).

Finally, substitute this Q-function into the un-regularized actor loss, we see that the result is equivalent to the $\lambda$-weighted actor loss:
\begin{align*}
   \E_{p(s) \pi(a \mid s)}\left[\log Q^*(s, a) \right] = & \E_{p(s) \pi(a \mid s)}\bigg[\log Q^\beta(s, a) + \underbrace{\log \beta(a \mid s) - \log \left((1 - \lambda_\text{KL}) \pi(a \mid s) + \lambda_\text{KL} \beta(a \mid s) \right)}_{\text{$\lambda$-weighted actor regularizer}}\bigg]
\end{align*}
\end{proof}

\section{Regularization for Goal-Conditioned Problems}
\label{appendix:goals}

Like single-task RL problems, goal-conditioned RL problems have also been approached with both one-step methods~\citep{ghosh2020learning, ding2019goal, sun2019policy} and critic regularization~\citep{chebotar2021actionable}.
In these problems, the aim is to learn a goal-conditioned policy $\pi(a \mid s, s_g)$ that maximizes the expected discounted sum of goal-conditioned rewards $r_g(s, a)$, where goals are sampled $s_g \sim p_g(s_g)$:
\begin{equation*}
    \max_\pi \E_{p_g(s_g)}\E_{\pi(\tau \mid s_g)}\left[\sum_{t=0}^\infty \gamma^t r_g(s_t, a_t) \right].
\end{equation*}
We will use the goal-conditioned reward function $r_g(s, a) = p(s' = s_g \mid s, a)$, which is defined in terms of the environment dynamics. In settings with discrete states, maximizing this reward function is equivalent to maximizing the sparse indicator reward function ($r_g(s, a) = \mathbbm{1}(s_g = s)$).

In this section, we show that one-step RL and critic regularization are equivalent for a certain goal-conditioned actor-critic method. Unlike our analysis in the single-task setting, this analysis here uses an existing method, C-learning~\citep{eysenbach2020c}. C-learning is a TD method that already makes use of the cross entropy loss for training the critic:
\begin{align*}
    \max_Q \; & (1 - \gamma) \E_{p(s, a, s')}\left[\log \frac{Q(s, a, s_g=s')}{Q(s, a, s_g=s') + 1} \right] + \gamma \E_{p(s, a) p_g(s_g)}\left[ y^{\pi, Q_t}(s, a, s_)  \log \frac{Q(s, a, s_g)}{Q(s, a, s_g) + 1} \right] \\
    & + \E_{p(s, a)p_g(s_g)}\left[\log \frac{1}{Q(s, a, s_g = s') + 1} \right],
\end{align*}
where $y^{\pi, Q_t}(s, a, s_g) = \E_{p(s' \mid s, a)\pi(a' \mid s', s_g)}\left[Q(s', a', s_g)\right]$ serves the role of the TD target.

The first two terms increase the Q-values while the last term decreases the Q-values. The actor is updated to maximize the Q-values. While this objective for the actor can be written in many ways, we will write it as maximizing a log ratio because it will allow us to draw a precise equivalence between actor and critic regularization:
\begin{equation*}
    \max_\pi \E_{p_g(s_g) p(s) \pi(a \mid s, s_g)} \left[ \log Q(s, a, s_g) \right]
\end{equation*}

We will now consider variants of C-learning that incorporate actor and critic regularization.

\paragraph{One-step RL.}
We will consider a variant of C-learning that resembles one-step RL~\citep{brandfonbrener2021offline}. The critic update will be similar to before, but the next-actions sampled for the TD updates will be sampled from the \emph{marginal} behavioral policy:
\begin{align*}
    \max_Q \; & (1 - \gamma) \E_{p(s, a, s')}\left[\log \frac{Q(s, a, s_g=s')}{Q(s, a, s_g=s') + 1} \right] + \gamma \E_{p(s, a) p_g(s_g)}\left[ y^{\beta, Q_t}(s, a, s_)  \log \frac{Q(s, a, s_g)}{Q(s, a, s_g) + 1} \right] \\
    & + \E_{p(s, a)p_g(s_g)}\left[\log \frac{1}{Q(s, a, s_g = s') + 1} \right],
\end{align*}
where $y^{\beta, Q_t}(s, a, s_g) = \E_{p(s' \mid s, a)\beta(a' \mid s')}[Q_t(s', a', s_g)]$. The actor update will be modified to include a reverse KL divergence:
\begin{equation}
    \max_\pi \E_{p(s)p_g(s_g)\pi(a \mid s, s_g)}\left[\log Q(s, a, s_g) + \log \beta(a \mid s) - \pi(a \mid s, s_g) \right]. \label{eq:gcrl-policy-regularized}
\end{equation}

Note that we are regularizing the policy to be similar to the average behavioral policy, $\beta(a \mid s)$. Compared to regularization towards a goal-conditioned behavioral policy $\beta(a \mid s, s_g)$, this choice gives the policy additional flexibility: when trying to reach goal $s_g$, it is allowed to take actions that were not taken by $\beta(a \mid s, s_g)$, as long as they were taken by the behavioral policy when trying to reach some other goal $s_g'$.

\paragraph{Critic regularization.}

To regularize the critic, we will modify the ``negative'' term in the C-learning objective to use actions sampled from the policy:
\begin{align}
    \max_Q \; & (1 - \gamma) \E_{p(s, a, s')}\left[\log \frac{Q(s, a, s_g=s')}{Q(s, a, s_g=s') + 1} \right] \\
    & + \gamma \E_{p(s, a) p_g(s_g)}\left[ y^{\pi, Q_t}(s, a, s_g)  \log \frac{Q(s, a, s_g)}{Q(s, a, s_g) + 1} \right] \\
    & + \E_{p(s)p_g(s_g), a \sim \pi(\cdot \mid s, s_g)}\left[\log \frac{1}{Q(s, a, s_g) + 1} \right]. \label{eq:gcrl-critic-regularized}
\end{align}

\subsection{Analysis for Goal-Conditioned Problems}

Like in the single-task setting, these two forms of regularization yield the same fixed points:
\begin{theorem} \label{thm:gcrl}
Let policy $\pi(a \mid s, s_g)$ be given, let $Q^\beta(s, a, s_g)$ be the Q-values for the marginal behavioral policy $\beta(a \mid s)$ and let $Q_r^\pi(s, a, s_g)$ be the critic obtained by the regularized critic update (Eq.~\ref{eq:gcrl-critic-regularized}). Then performing regularized policy updates (Eq.~\ref{eq:gcrl-policy-regularized}) using the behavioral critic is equivalent to the un-regularized policy objective using the regularized critic:
\begin{equation*}
    \E_{\pi(a \mid s, s_g)}\left[\log Q^\beta(s, a, s_g) + \log \beta(a \mid s) - \log \pi(a \mid s, s_g) \right] = \E_{\pi(a \mid s, s_g)}\left[\log Q_r^\pi(s, a, s_g)\right]
\end{equation*}
for all states $s$ and goals $s_g$.
\end{theorem}

\begin{proof}
We start by determining the fixed point of critic-regularized C-learning. Like in the single-task setting, the C-learning objective resembles a weighted-classification problem, so we can write down the Bayes' optimal classifier as 
\begin{equation*}
    \frac{Q(s, a, s_g)}{Q(s, a, s_g) + 1} = \frac{((1 - \gamma)p(s' = s_g \mid s, a) + \gamma p(s = s_g)y(s', s_g))\beta(a \mid s)}{((1 - \gamma)p(s' = s_g \mid s, a) + \gamma p(s = s_g)y(s', s_g))\beta(a \mid s) + p(s_g)\pi(a \mid s, s_g)}.
\end{equation*}
Solving for $Q(s, a, s_g)$ on the left hand side, the optimal value for $Q(s, a, s_g)$ is given by
\begin{equation*}
    Q(s, a, s_g) = ((1 - \gamma) p(s' = s_g \mid s, a) + \gamma p(s = s_g)y(s', s_g)) \frac{\beta(a \mid s)}{\pi(a \mid s, s_g)}
\end{equation*}
This tells us what each critic-regularized C-learning update does.

To analyze these updates, we define $\tilde{Q}(s, a, s_g) \triangleq Q(s, a, s_g) \frac{\pi(a \mid s, s_g)}{\beta(a \mid s)}$. Then these updates can be written using $\tilde{Q}(s, a, s_g)$ as
\begin{equation*}
    \tilde{Q}(s, a, s_g)\frac{\beta(a \mid s)}{\pi(a \mid s, s_g)} = \left((1 - \gamma) p(s' = s_g \mid s, a) + \gamma \E_{p(s' \mid s, a) \pi(a' \mid s', s_g)}\left[\tilde{Q}(s', a', s_g)\frac{\beta(a' \mid s')}{\pi(a' \mid s', s_g)} \right] \right)\frac{\beta(a \mid s)}{\pi(a \mid s, s_g)}.
\end{equation*}
These updates can be simplified to
\begin{equation*}
    \tilde{Q}(s, a, s_g) = (1 - \gamma) p(s' = s_g \mid s, a) + \gamma \E_{p(s' \mid s, a) \beta(a' \mid s')}\left[\tilde{Q}(s', a', s_g)\right].
\end{equation*}
Like before, the ratio $\frac{\beta(a' \mid s')}{\pi(a' \mid s', s_g)}$ inside the expectation acts like an importance weight. Thus, the regularized critic updates are equivalent to perform policy evaluation on $\tilde{Q}(s, a, s_g)$. Note that this is estimating the probability that the \emph{average} behavioral policy $\beta(a \mid s)$ reaches goal $s_g$; this is \emph{not} the probability that a goal-directed behavioral policy $\beta(a \mid s, s_g)$ reaches the goal.

Finally, we translate these convergence results for $\tilde{Q}(s, a, s_g)$ into convergence results for $Q(s, a, s_g)$. Written in terms of the original Q-values, we see that the optimal critic for the regularized critic update is
\begin{equation*}
    Q^*(s, a, s_g) = \tilde{Q}^*(s, a, s_g) \frac{\beta(a \mid s)}{\pi(a \mid s, s_g)} = Q^{\beta(\cdot \mid \cdot)}(s, a, s_g) \frac{\beta(a \mid s)}{\pi(a \mid s, s_g)}.
\end{equation*}

Thus, critic regularization implicitly regularizes the actor objective so that it is the same objective as one-step RL:
\begin{align*}
    & \E_{p(s), s_g \sim p(s), \pi(a \mid s, s_g)}\left[ \log Q^*(s, a, s_g) \right] \\
    &= \E_{p(s), s_g \sim p(s), \pi(a \mid s, s_g)}\left[ \log Q^{\beta(\cdot \mid \cdot)}(s, a, s_g) + \log \beta(a \mid s) - \log \pi(a \mid s, s_g) \right].
\end{align*}
\end{proof}

\section{Regularization for Example-based Control Problems}
\label{appendix:examples}

While specifying tasks in terms of reward functions is standard for MDPs, it can be difficult for real-world applications of RL. So, prior work has looked at specifying tasks by goal states (as in the previous section) or sets of states representing good outcomes~\citep{pinto2016supersizing, tung2018reward,  fu2018variational}. In addition to requiring more flexible and user-friend forms of task specification, these algorithms targeted at real-world applications often demand regularization. In the same way that prior goal-conditioned RL algorithms have employed critic regularization, so too have prior example-based control algorithms~\citep{singh2019end, hatch2022example}. In this section, we extend our analysis to regularization of an example-based control algorithm. Again, we will show that a certain form of critic regularization is equivalent to regularizing the actor.

We first define the problem of example-based control~\citep{fu2018variational}. In these problems, the agent is given a small collection of states $s \sim p_e(s)$, which are examples of successful outcomes. The aim is to learn a policy $\pi(a \mid s)$ that maximizes the probability of reaching a success state:
\begin{equation*}
    \max_\pi \E_{p(s_g)}\E_{\pi(\tau \mid s_g)}\left[\sum_{t=0}^\infty \gamma^t p_e(s_t) \right].
\end{equation*}
Note that this objective function is exactly equivalent to a reward-maximization problem, with a reward function $r(s, a) = p_e(s_t)$.

In this section, we show that one-step RL and critic regularization are equivalent for a certain example-based control algorithm. Unlike our analysis in the single-task setting, this analysis here uses an existing method, RCE~\citep{eysenbach2021replacing}. RCE is a TD method that already makes use of the cross entropy loss for training the critic:\begin{align*}
    \max_Q \; & (1 - \gamma) \E_{p_e(s)\beta(a \mid s)}\left[\log \frac{Q(s, a)}{Q(s, a) + 1} \right] + \E_{p(s, a)}\left[ \gamma y^{\pi, Q_t}(s, a)  \log \frac{Q(s, a)}{Q(s, a) + 1} + \log \frac{1}{Q(s, a) + 1} \right],
\end{align*}
where $y^{\pi, Q_t}(s, a) = \E_{p(s' \mid s, a)\pi(a' \mid s')}[Q(s', a')]$ serves the role of the TD target.
The first two terms increase the Q-values while the last term decreases the Q-values. 
The actor is updated to maximize the Q-values. While this objective for the actor can be written in many ways, we will write it as maximizing a log ratio because it will allow us to draw a precise equivalence between actor and critic regularization:
\begin{equation*}
    \max_\pi \E_{p(s)\pi(a \mid s)} \left[ \log Q(s, a) \right]
\end{equation*}
We will now consider variants of RCE that incorporate actor and critic regularization.

\paragraph{One-step RL.}
We will consider a variant of RCE that resembles one-step RL~\citep{brandfonbrener2021offline}. The critic update will be similar to before, but the next-actions sampled for the TD updates will be sampled from the behavioral policy:
\begin{align*}
    \max_Q \; & (1 - \gamma) \E_{p_e(s)\beta(a \mid s)}\left[\log \frac{Q(s, a)}{Q(s, a) + 1} \right] + \E_{p(s, a)}\left[ \gamma y^{\beta, Q_t}(s, a)  \log \frac{Q(s, a)}{Q(s, a) + 1} + \log \frac{1}{Q(s, a) + 1} \right],
\end{align*}
where $y^{\beta, Q_t}(s, a) = \E_{p(s' \mid s, a)\beta(a' \mid s')}[Q(s', a')]$. The actor update will be modified to include a reverse KL divergence:
\begin{equation}
    \max_\pi \E_{p(s), \pi(a \mid s)}\left[\log Q(s, a) + \log \beta(a \mid s) - \pi(a \mid s) \right]. \label{eq:rce-policy-regularized}
\end{equation}

\paragraph{Critic regularization.}

To regularize the critic, we will modify the ``negative'' term in the RCE objective to use actions sampled from the policy:
\begin{align}
    (1 - \gamma) \E_{p_e(s)\beta(a \mid s)}\left[\log \frac{Q(s, a)}{Q(s, a) + 1} \right] + \E_{p(s, a), a^- \sim \pi(\cdot \mid s)}\left[ \gamma y^{\pi, Q_t}(s, a)  \log \frac{Q(s, a)}{Q(s, a) + 1} + \log \frac{1}{Q(s, a^-) + 1} \right], \label{eq:rce-critic-regularized}
\end{align}

\subsection{Analysis for Example-based Control Problems}

Like in the single-task setting, these two forms of regularization yield the same fixed points:
\begin{theorem} \label{thm:rce}
Let policy $\pi(a \mid s)$ be given, let $Q^\beta(s, a)$ be the Q-values for the behavioral policy $\beta(a \mid s)$ and let $Q_r^\pi(s, a)$ be the critic obtained by the regularized critic update (Eq.~\ref{eq:rce-critic-regularized}). Then performing regularized policy updates (Eq.~\ref{eq:rce-policy-regularized}) using the behavioral critic is equivalent to the un-regularized policy objective using the regularized critic:
\begin{equation*}
    \E_{\pi(a \mid s)}\left[\log Q^\beta(s, a) + \log \beta(a \mid s) - \log \pi(a \mid s) \right] = \E_{\pi(a \mid s)}\left[\log Q_r^\pi(s, a)\right]
\end{equation*}
for all states $s$.
\end{theorem}

\begin{proof}
We start by determining the fixed point of critic-regularized RCE. Like in the single-task setting, The RCE objective resembles a weighted-classification problem, so we can write down the Bayes' optimal classifier as 
\begin{equation*}
    \frac{Q(s, a)}{Q(s, a) + 1} = \frac{((1 - \gamma)p_e(s) + \gamma y^{\pi, Q_t}(s, a))\beta(a \mid s)}{((1 - \gamma)p_e(s) + \gamma y^{\pi, Q_t}(s, a))\beta(a \mid s) + \pi(a \mid s)}.
\end{equation*}
Solving for $Q(s, a)$ on the left hand side, the optimal value for $Q(s, a)$ is given by
\begin{equation*}
    Q(s, a) = ((1 - \gamma) p_e(s) + \gamma y^{\pi, Q_t}(s, a)) \frac{\beta(a \mid s)}{\pi(a \mid s)}
\end{equation*}
This tells us what each critic-regularized RCE update does.

To analyze these updates, we define $\tilde{Q}(s, a) \triangleq Q(s, a) \frac{\pi(a \mid s)}{\beta(a \mid s)}$. Then these updates can be written using $\tilde{Q}(s, a)$ as
\begin{equation*}
    \tilde{Q}(s, a)\frac{\beta(a \mid s)}{\pi(a \mid s)} = \left((1 - \gamma) p_e(s) + \gamma \E_{p(s' \mid s, a) \pi(a' \mid s')}\left[\tilde{Q}(s', a')\frac{\beta(a' \mid s')}{\pi(a' \mid s')} \right] \right)\frac{\beta(a \mid s)}{\pi(a \mid s)}.
\end{equation*}
These updates can be simplified to
\begin{equation*}
    \tilde{Q}(s, a) = (1 - \gamma) p_e(s) + \gamma \E_{p(s' \mid s, a) \beta(a' \mid s')}\left[\tilde{Q}(s', a')\right].
\end{equation*}
Like before, the ratio $\frac{\beta(a' \mid s')}{\pi(a' \mid s')}$ inside the expectation acts like an importance weight. Thus, the regularized critic updates are equivalent to perform policy evaluation on $\tilde{Q}(s, a)$.

Finally, we translate these convergence results for $\tilde{Q}(s, a)$ into convergence results for $Q(s, a)$. Written in terms of the original Q-values, we see that the optimal critic for the regularized critic update is
\begin{equation*}
    Q^*(s, a) = \tilde{Q}^*(s, a) \frac{\beta(a \mid s)}{\pi(a \mid s)} = Q^{\beta}(s, a) \frac{\beta(a \mid s)}{\pi(a \mid s)}.
\end{equation*}

Thus, critic regularization implicitly regularizes the actor objective so that it is the same objective as one-step RL:
 \begin{align*}
    & \E_{p(s), \pi(a \mid s)}\left[ \log Q^*(s, a) \right] = \E_{p(s), \pi(a \mid s)}\left[ \log Q^{\beta}(s, a) + \log \beta(a \mid s) - \log \pi(a \mid s) \right].
\end{align*}
\end{proof}

\section{Additional Experiments}

\subsection{Is classifier actor critic a good model of practically-used RL methods?}
\label{sec:good-model}

To study whether classifier actor critic is an accurate model for practically-used RL methods, we extended Fig.~\ref{fig:cql-sarsa} by adding three additional methods: classifier actor critic, classifier actor critic with actor regularization (Equations~\ref{eq:actor} and~\ref{eq:critic}). Comparing Q-learning with classifier actor critic, we see that both yield reward-maximizing policies (in line with Lemma~\ref{lemma:1}), though these policies are different (i.e., they perform symmetry breaking in different ways). Comparing one-step RL with actor-regularized classifier actor critic, we observe that both methods take the same three actions at the states within the blue box, states where we expect regularization to have a large effect. Outside this blue box, these two methods occasionally take different actions.
Similarly, comparing CQL to critic-regularized classifier actor-critic, we observe that the methods take the same actions within the blue box, but occasionally take different actions outside the blue box. In line with Theorem~\ref{thm:main}, classifier actor-critic with critic regularization produces the exact same policy as classifier actor-critic with actor regularization.
Taken together, these results provide empirical backing for our theoretical results, while also showing that classifier actor critic is only an approximate model of practical algorithms, not a perfect model.

\begin{figure*}[t]
    \centering
    \begin{subfigure}[b]{0.25\textwidth}
        \includegraphics[width=\linewidth]{figures/cql_0_gridworld.pdf}
        \caption{ Q-learning}
    \end{subfigure}
    \hfill
     \begin{subfigure}[b]{0.25\textwidth}
        \includegraphics[width=\linewidth]{figures/sarsa_gridworld.pdf}
        \caption{ One-step RL}
    \end{subfigure}
    \hfill
    \begin{subfigure}[b]{0.25\textwidth}
        \includegraphics[width=\linewidth]{figures/cql_10_gridworld.pdf}
        \caption{ CQL($\lambda=10$)}
    \end{subfigure}
    \begin{subfigure}[b]{0.25\textwidth}
        \includegraphics[width=\linewidth]{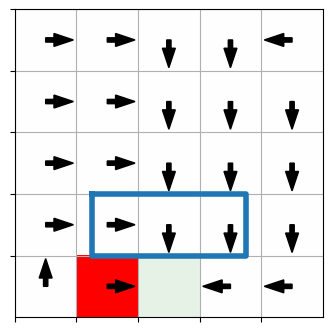}
        \caption{\centering \phantom{....} classifier actor critic \phantom{..........} (no regularization)}
    \end{subfigure}
    \hfill
    \begin{subfigure}[b]{0.25\textwidth}
        \includegraphics[width=\linewidth]{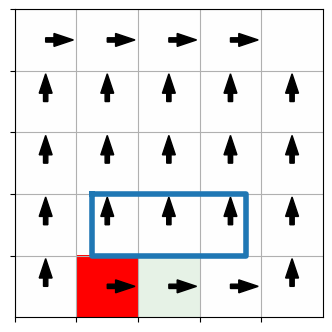}
        \caption{\centering \phantom{...} classifier actor critic \phantom{........} (actor regularization)}
    \end{subfigure}
    \hfill
     \begin{subfigure}[b]{0.25\textwidth}
        \includegraphics[width=\linewidth]{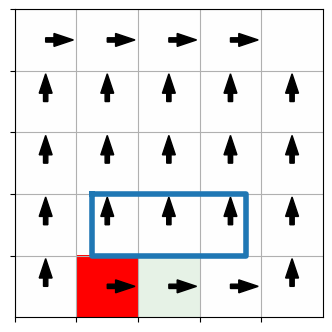}
        \caption{\centering \phantom{..} classifier actor critic \phantom{........} (critic regularization)}
    \end{subfigure}
    \vspace{-0.5em}
    \caption{Extension of Fig.~\ref{fig:cql-sarsa} with additional figures for classifier actor critic and regularized variants. As expected, classifier actor critic produces a reward-maximizing policy (Lemma~\ref{lemma:1}), as does Q-learning. In line with Theorem~\ref{thm:main}, critic-regularized and actor-regularized classifier actor critic product the same policies. 
    Comparing (b) one-step RL (e) actor-regularized classifier actor critic, we see that the methods produce the same three actions within the blue box (states where we expect regularization to be important), but can produce different actions in states outside the blue box. The comparisons between (c) CQL and (d) critic-regularized classifier actor critic are similar, supporting the claim that classifier actor critic is only an approximate (no a perfect) model of practically-used offline RL methods.
    }
    \label{fig:cac}
\end{figure*}

\section{Experimental Details}
\label{appendix:details}

\subsection{Tabular experiments}

\paragraph{Implementing critic regularization for classifier actor critic.}
The objective for critic regularization in contrastive actor critic (Eq.~\ref{eq:critic-regularized}) is nontrivial to optimize because of the cyclic dependency between the policy and the critic: simply alternating between optimizing the actor and the critic does not converge. In our experiments, we update the critic using an exponential moving average of the policy, as proposed in~\citet{wen2021characterizing}.  %
We found that this decision was sufficient for ensuring convergence. When applying CQL in the tabular setting (Figures~\ref{fig:cql-sarsa} and~\ref{fig:histogram}), we did not do this because soft value iteration represents the policy implicitly in terms of the value function.

\paragraph{Fig.~\ref{fig:gridworld} \emph{(left)}} The initial state and goal state are located in opposite corners. The reward function is $+1$ for reaching the goal and 0 otherwise. We use a dataset of 20 trajectories, 50 steps each, collected by a random policy. We use $\gamma = 0.95$ and train for 20k full-batch updates, using a learning rate of 1e-2. The Q table is randomly initialized using a standard normal distribution.

\paragraph{Fig.~\ref{fig:gridworld} \emph{(center)}} The initial state and goal state are located in adjacent corners. The goal state has a reward of +3.5, the states between the initial state and goal state have a reward +1, and all other states (including the initial state) have a reward of +2. We use a dataset of 20 trajectories, 50 steps each, collected by a random policy. We use $\gamma = 0.95$ and train for 20k full-batch updates, using a learning rate of 1e-2. The Q table is randomly initialized using a standard normal distribution.

\paragraph{Fig.~\ref{fig:gridworld} \emph{(right)}} The initial state and goal state are located in adjacent corners. The reward is +0.01 at the goal state and 0 otherwise. We use a dataset of 1 trajectories with 10 steps, which traces the following path: 
\begin{equation*}
[(0, 0), (1, 0), (1, 1), (1, 2), (1, 3), (1, 4), (0, 4), (0, 4), (0, 4), (0, 4)].
\end{equation*}
\!We use $\gamma = 0.95$ and train for 10k full-batch updates, using a learning rate of 1e-2. The Q table is randomly initialized using a standard normal distribution.

\paragraph{Fig.~\ref{fig:cql-sarsa}} There is a bad state (reward of $-10$) next to the optimal state (reward of $+1$), so the behavioral policy navigates away from the optimal state.  We generate 10 trajectories of length 100 from a uniform random policy. We use $\gamma = 0.95$ and train each method for 10k full-batch updates. The Q table is randomly initialized using a standard normal distribution. One-step RL performs SARSA updates while CQL performs soft value iteration (as suggested in the CQL paper).

\paragraph{Fig.~\ref{fig:histogram}} We generate 100 random variants of Fig.~\ref{fig:cql-sarsa} by randomly sampling the high-reward state and low-reward state (without replacement). The datasets are generated in the same way.

\paragraph{Fig.~\ref{fig:when}}
We use the same environment and dataset as in Fig.~\ref{fig:cql-sarsa}, but train the CQL agent with varying values of $\lambda$, each with 5 random seeds. We train the one-step RL agent for 5 random seeds. For each point on the X axis of Fig.~\ref{fig:when}, we compare compute $5 \times 5$ pairwise comparisons and report the mean and standard deviation.

\subsection{Continuous control experiments}
For the experiments in Figures~\ref{fig:q_vals} and~\ref{fig:mse-loss}, we used the implementation of one-step RL (reverse KL) and CQL provided by~\citet{hoffman2020acme}. We choose this implementation because it is well tuned and uses similar hyperparameters for the two methods. As mentioned in the main text, the only change we made to the implementation was adding the twin-Q trick to one-step RL, such that it matched the critic architecture used by CQL. We did not change any of the other hyperparameters, including hyperparameters controlling the regularization strength.

\end{document}